\documentclass[11pt]{article}
\usepackage{amsmath}
\usepackage{amsthm}
\usepackage{amssymb}
\usepackage{dsfont}
\usepackage[utf8]{inputenc}
\usepackage{graphicx}
\usepackage{subcaption}
\usepackage{float}
\usepackage{varioref}

\usepackage{color}
\usepackage{mathrsfs}
\usepackage{float}
\usepackage{soul}
\usepackage{url}   
\usepackage{setspace}

\usepackage[
  style=authoryear,
  natbib=true,
  maxcitenames=2,
  maxbibnames=99,
  backend=biber
]{biblatex}
\renewbibmacro{in:}{}

\usepackage{adjustbox} 
\usepackage{breqn}
\usepackage{comment}
\usepackage{hyperref}
\usepackage{nameref}
\usepackage{cleveref}
\usepackage{booktabs}
\usepackage{graphicx}
\usepackage{multirow}

\allowdisplaybreaks 

\newtheorem{theorem}{Theorem}[section]

\newtheorem{corollary}{Corollary}[section]

\newtheorem{definition}{Definition}[section]

\newtheorem{proposition}{Proposition}[section]
\newtheorem{remark}{Remark}[section]

\makeatother

\newcommand{\X}{\mathbf{X}}

\newcommand{\N}{\mathcal{N}}

\newcommand{\p}{\mathbb{P}}

\newcommand{\A}{A}

\newcommand{\z}{\mathbf{z}}

\newcommand{\y}{\mathbf{y}}
\newcommand{\s}{\mathcal{S}}

\newcommand{\indep}{\mathrel{\perp\!\!\!\perp}}

\crefname{appendix}{Appendix}{Appendices}

\hypersetup{colorlinks=true,
            citecolor=blue,
            filecolor=blue,
            linkcolor=blue,
            urlcolor=blue}

\numberwithin{equation}{section}

\usepackage[left=2cm,top=2cm,bottom=2.5cm,right=2cm,nohead,foot=2cm]{geometry}
\begin{document}

\title{Shapley meets Rawls: an integrated framework for measuring and explaining unfairness}

\author{ Fadoua Amri-Jouidel\thanks{ University Mohammed VI Polytechnic, Rabat, Salé. Corresponding author: Fadoua.JOUIDEL-AMRI@um6p.ma} \\
\textsc{UM6P} \\
\textsc{Airess}
\and Emmanuel Kemel\thanks{$^{{}}$\textsc{Cnrs}, HEC Paris } \\
\textsc{Greghec, Crns} \\ \textsc{HEC Paris}
\and St\'{e}phane
Mussard\thanks{$^{{}}$\textsc{Univ. N\^{i}mes Chrome}, Rue du Dr Georges Salan, 30000 N\^{i}mes ; Email: stephane.mussard@unimes.fr ; Research fellow \textsc{Liser} Luxembourg and University Mohammed VI Polytechnic.} \\
\textsc{Univ. N\^{i}mes Chrome}\\ \textsc{Um6p, Airess} } 

\date{}
\maketitle

\begin{abstract}
Explainability and fairness have mainly been considered separately, with recent exceptions to explain the sources of unfairness. 
This paper shows that the Shapley value can be used to both define and explain unfairness, under standard group fairness criteria. This offers an integrated framework to estimate and derive inference on unfairness as well as the features that contribute to it. 
Our framework can also be extended from Shapley values to the family of Efficient-Symmetric-Linear (ESL) values, some of which offer more robust definitions of fairness, and shorter computation times.
An illustration is run on  the Census Income dataset from the UCI Machine Learning Repository. Our approach shows that ``Age",  ``Number of hours" and ``Marital status" generate gender unfairness, using shorter computation time  than traditional Bootstrap tests.
\end{abstract}


\medskip

\textbf{Declarations of competing interest:} None.

\medskip

\textbf{Data statement:} Code and data are available on our github: \href{https://github.com/anonymous03032026/fairles_shapley.git}{link}

\maketitle

\break

\setlength{\baselineskip}{17pt}

\section{Introduction}\label{intro}


Artificial Intelligence (AI) is increasingly being integrated to assist or automate decision making across  various domains such as healthcare (\cite{obermeyer2019dissecting}), hiring (\cite{kochling2020discriminated}), credit (\cite{lee2021algorithmic}), justice (\cite{zavrvsnik2021algorithmic}) and public policy (\cite{pencheva2020big}). While these automated decisions are cost effective and more systematic than human decisions, they also raise ethical concerns. Notorious examples include the COMPAS algorithm (\cite{flores2016false}), accused of producing more false-positives for recidivism among ethnic minorities, and the Amazon recruitment algorithm reported to generate lower acceptance rates for women (\cite{dastin2022amazon}).

These examples illustrate the need for tools to monitor  unfairness across groups. Monitoring unfair decisions also requires explainability. At the individual level, those who receive a negative outcome deserve an explanation about the reasons for this outcome. At the aggregate level, it is important to quantify how features (\textit{i.e.}, variables) contribute to the outcome of an algorithm in order to assess possible trade-offs between fairness and performance. 
In line with much of the literature, we focus on algorithms that produce categorical outcomes, which we hereafter refer to as classifiers. We study the fairness and explainability of classifiers, two notions that, although complementary, have essentially been considered separately. 

\textit{Group fairness} evaluates whether a classifier yields 
 different outcome distributions across groups defined by sensitive features (\textit{e.g.}, gender, ethnicity).  For instance, (\cite{klare2012face}) evaluate the outcomes of multiple face-recognition classifiers and report disparities between white and non-white women.
The algorithmic fairness literature proposes criteria to either audit  predictions (\cite{barocas2023fairness,charpentier2024insurance, hurlin2024fairness}) and detect unfairness, or to impose ``fair learning'', \textit{i.e.,} to incorporate fairness constraints during classifier training (\cite{Hardt2016}), consequently allowing for a trade-off between efficiency of classifiers and fairness as is done in operational research \citep{DEBOCK2024}. In both cases, the criteria 
consist  of a group level equality condition on certain metrics. Specifically, a given statistical metric related to the outcome distribution should be equal across groups.

Three main fairness criteria are generally considered in the literature (see Section 2.1 for details). Independence \eqref{IND}, also known as statistical parity (\cite{dwork2012fairness}), requires that the expected outcome be equal across groups. In this case, the metric is the group expected outcome, and the fairness criterion is the equality of the metric between groups. This criterion can result in groups with different levels of merit receiving the same expected outcome, which may be perceived as unfair. Separation \eqref{SEP}, also known as equalized odds \citep{Hardt2016}, requires that error distributions be equal across groups. In other words, it ensures that, when predicting a given outcome, the classifier yields equal true and false positive rates across groups.
Sufficiency \eqref{SUF}, also called conditional use accuracy (\cite{berk2021fairness}), equalizes reliability across groups. It ensures that for both positive and negative predicted outcomes, the reliability is equal across groups.
 Impossibility results show that one cannot generally satisfy these three criteria simultaneously (\cite{kleinberg2016inherent,chouldechova2017fair}). Choosing one of them is a matter of ethical values. 
 AI fairness thus places regulators and users in a Rawlsian situation where a commonly agreed rule of fairness must be adopted (\cite{rawls1971theory}) and then implemented by institutions.

\textit{Explainable AI (XAI)} is another field in AI ethics that enhances the transparency of algorithms, making their decision-making processes more interpretable and understandable. A particular attention is dedicated to the measurement of the importance of each feature to the decision making outcomes. A popular approach for measuring feature contributions is the partial dependence plot (PDP) (\citealt{molnar2020interpretable}). Such a plot shows how a model's predicted outcome changes as one feature varies across its range, averaging over the observed values of all the other features. Another approach builds on the Shapley value \citep{shapley1953}, which attributes a score of importance to each feature. It is an agnostic \textit{attribution method}, meaning that it is independent of the type of algorithm used for classification tasks (classification of images, texts, etc., see \citealt{castro2009,sun2011,ancona2019, arenas2021,Condevaux2022} for applications). Although more computationally demanding, Shapley values offer two key advantages over PDP. First, by considering all possible combinations of variables, they account for the exhaustive list of possible interactions. Second, they allow for a linear decomposition of the feature contributions, which facilitates interpretation. To overcome the exponential computational cost,  \citet{Condevaux2022} proposed alternative attribution methods such as Fair-Efficient-Symmetric-Perturbation (FESP) and the Equal Surplus value \citep{Driessen1991} which are characterized by  linear complexity, making them convenient alternatives to the Shapley value.

\textit{Group fairness and explainability} were considered separately in the literature until recently, \citet{begley2020explainability} and \citet{pelegrina2024explaining} propose to explain unfairness using Shapley values. The method proposed by \citet{pelegrina2024explaining} applies to different fairness criteria. For example, it could be applied to equal opportunity (a relaxation of \eqref{SEP}) between genders. The difference in true positive rates (TPR) between men and women can be computed, then  the contribution of each feature to the gender gap can be measured in the form of Shapley value.

The present paper pushes the connection between fairness criteria and Shapley explainability yet further. We highlight that the fairness criterion itself can be expressed as a difference between Shapley values. Specifically, we introduce the notion of ``group-Shapley" that measures the contribution of the observations corresponding to a given group (\textit{e.g.}, ``men") of a sensitive feature (\textit{e.g.}, gender) to the overall value of an outcome metric \textit{(e.g.}, accuracy,  acceptance rate, or false positive rate...). Any unfairness criterion which is defined as group-differences in outcome metrics can also be defined as differences of group-Shapley values. Besides the conceptual elegance of connecting formally Shapley values and fairness criteria, this approach allows to build on the recursive nature of Shapley decompositions. 

Concretely, our method proceeds in two stages: (i) in the first stage Shapley, we measure unfairness as differences in group-Shapley values across groups; (ii) in the second stage Shapley, we  recursively decompose each group-Shapley value over the features.
This approach first breaks inequality between groups, then breaks inequality within each group over all the features. It corresponds to a standard approach employed in the income inequality literature (\textit{e.g.}, \cite{chantreuil2011}). This can be particularly useful when the sensitive feature contains more than two levels. 
Combined with linearity, recursivity provides convenient properties for explaining the contributions of features to unfairness. 
In our framework, the contribution of each feature to the differences of group-Shapley values is measured as differences of Shapley contributions to each group. 

This approach of fairness has two noticeable advantages. The first relates to interpretation, the second relates to inference.  
In terms of interpretation, for each feature and each group, the Shapley contribution informs about the impact of the feature on the outcome of the group. Then, the Shapley contribution of a feature to the gap is defined as a difference of group-Shapley values. This difference can be positive or negative. Partial dependence plots describe how average model predictions vary with a feature. Our Shapley decomposition assigns each feature a signed contribution to the group disparity, with contributions summing to the overall gap. When the groups are ordered so that this gap is positive, positive contributions widen the disparity, whereas negative contributions offset it.
Suppose that the outcome is measured by the acceptance rate, and that feature ``income'' contributes positively to it for both genders, but to a larger extent for men. Mitigating the role of income will reduce the gap by deteriorating the acceptance rate for both men and women: an inequality reducing, but Pareto deteriorating solution. Suppose instead that another variable, \textit{e.g.}, marital status, has a positive impact for men and a negative impact for women. Here, mitigating the role of this variable reduces the gap by decreasing the acceptance rate of men and increasing the acceptance rate of women. A social planner who is unfairness averse but does not want to deteriorate the outcome of the disadvantaged group may prefer the second solution. The second advantage concerns with inference. Once the asymptotic properties of the fairness criteria have been defined, they can be applied to the second-stage Shapley. In particular, when the first-stage Shapley value is computed, a second-stage Shapley value is made along features
to explain their importance to each group-Shapley, in the same vein as \citet{pelegrina2024explaining}. However, in our framework, the Shapley value is also used for a second statistical test, \textit{i.e.}, to test whether the contribution of a particular feature is the same along groups (\textit{e.g.}, contribution of education for men \textit{vs}. women). Consequently, the two-stage Shapley provides a test to assess whether a feature contributes significantly to the difference in group-Shapley values and is precisely shown to be asymptotically normal.

Additional results generalize our approach and allow for fast and robust implementation. 
The main limitation of the Shapley value is its exponential time complexity. Also, it can be found to possibly attribute too much importance to irrelevant variables (\cite{marques2024explainability}). Faster and more robust attribution procedures can be needed. We extend our results to the more general Efficient-Linear-Symmetric (ESL) family of values, which contains the Shapley value and the Equal Surplus value, among others. The Equal Surplus value has a linear time complexity, whereas most of the other members of the ESL family exhibit exponential time complexity.
Considering all ESL values enables a more robust attribution procedure based on majority voting \footnote{ Majority voting is well established in social choice theory \citep{young1988condorcet}, and was  later  adopted in machine learning to combine multiple classifier decisions \citep{battiti1994democracy,lam1997application, mienye2022survey}.}. Taking five ESL values, we compute feature contributions under each member separately, and classify a feature as unfair (or fair) when a majority of members indicate unfairness (or fairness).

We illustrate our framework on the Census Income dataset from the UCI Machine Learning Repository \citep{census_income_20}, where ``men" and ``women" constitute the two groups of the sensitive feature. We show that all five ESL values detect group unfairness (first-stage ESL). In terms of explainability, most  ESL values show that  ``Age" and  ``Hours/Week" penalize women,  whereas  ``Marital status"  penalizes men, thereby contributing to gender unfairness (two-stage ESL). Regarding computation time, the empirical application illustrates the tractability of our framework. Inference derived from our asymptotic formula provides results that are consistent with bootstrapping, but with shorter computation time. 

The paper is organized as follows. Group fairness and Shapley explainability are presented in Section \ref{prelims}, which can be skipped by readers familiar with this literature. Section \ref{FER} describes the  equivalence between group fairness and explainability with the proposed two-stage Shapley. Section \ref{ULES} presents the generalization to all ESL values.  
In  Section \ref{EXAMPLE}, an application is performed on Census Income data. Section \ref{conclusion} concludes the paper. 

\section{Group fairness and explainability: separate literature streams}\label{prelims}

We briefly summarize two seemingly independent literature streams. The first focuses on group fairness  criteria for classifiers with respect to sensitive features (\textit{e.g.}, gender, ethnicity). The second examines attribution methods used for explainability. These methods assign to each feature a contribution to a specific metric (\textit{e.g.}, contribution to the classification probability or to the goodness of fit of the classifier).

\subsection{Group fairness criteria}\label{sect:group_fairness}
 
Group fairness criteria fall into three categories, namely independence \citep{dwork2012fairness}, separation \citep{Hardt2016} and sufficiency \citep{berk2021fairness}.\footnote{See  \citet{Charpentier2023} for a comprehensive literature review.} Each criterion addresses a particular aspect of the classifier outcome to ensure equal treatment across groups defined by the sensitive feature. We introduce these criteria for a binary classifier. Let  $(X,A,Y)$ be a random tuple with distribution $\mathbb{P}$, where 
$Y \in \{0,1\}$ is the actual observed outcome used for fairness assessment, $X \in \mathcal{X}\subseteq\mathbb{R}^N$ is a collection of features, and $\A=\{0,1\}$ is a binary sensitive feature. We consider a binary classification function  $f_t$ that maps  $\mathcal{X}\times \{0,1\}$ to $\{0,1\}$, based on a given threshold $t \in(0,1)$: 

\[
\hat{Y} = f_t(X, A) =
\begin{cases} 
1 & \text{if } \mathbb{P}[Y = 1 \mid X , A ] \geq t \\
0 & \text{if } \mathbb{P}[Y = 1 \mid X , A ] < t
\end{cases}
\]

In what follows, $\indep$ denotes independence or conditional independence between random variables. The three following definitions are given by \citet{barocas2023fairness}.

\begin{definition}\textbf{\emph{Independence:}}  For a binary classifier $f_t$, independence corresponds to $\hat{Y}\indep A$, or equivalently,
\begin{equation}\tag{IND}\label{IND}
\p[\hat{Y} = 1 \ | A = 0] = \p[\hat{Y} = 1 \ | A =1] 
\end{equation}
where $\p[\hat{Y} = 1 \ | A = g]$ corresponds to the selection rate (SR) for group g.

\end{definition}
The independence criterion, also called statistical parity (\cite{dwork2012fairness}) or demographic parity (\cite{kusner2017counterfactual}), requires that the probability of receiving a favorable outcome (\textit{e.g.,}  loan approval, being hired)  be equal across groups. 
However, enforcing independence can yield situations
in which groups with different levels of ``merit" (\textit{e.g.,} qualifications) receive the same expected outcome. To address this issue, \citet{corbett2017algorithmic}  propose conditional statistical parity as a relaxation of demographic parity.

\begin{definition}\textbf{\emph{Separation:}} For a binary classifier $f_t$, separation corresponds to $\hat{Y} \indep A|Y$, or equivalently,
\begin{equation}\tag{SEP}\label{SEP}
\begin{aligned}
    &\p[\hat{Y} = 1 \ | \ Y = 1, A = 0] = \p[\hat{Y} = 1 \ | \ Y = 1, A = 1]   \\
    &\p[\hat{Y} = 1 \ | \ Y = 0, A = 0] = \p[\hat{Y} = 1 \ | \ Y = 0, A = 1] 
\end{aligned}
\end{equation}
where $\mathbb{P}(\hat Y = 1 \mid Y = 1, A=g)$ and 
$\mathbb{P}(\hat Y = 1 \mid Y = 0, A=g)$ correspond to the true positive rate (TPR) 
and false positive rate (FPR) for group $g$, respectively.

\end{definition}
The separation criterion, also referred  to as equalized odds (\cite{Hardt2016}), focuses on balancing error rates across groups. In other words, it ensures that the classifier is invariant to the sensitive feature when predicting a given class, with respect to both true positive and false positive rates. The  fair learning literature  also refers to the equal opportunity criterion (\cite{Hardt2016}), a relaxation of  separation that requires equality of true positive rates across groups. This criterion formalizes that, conditionally on the true labels (\textit{i.e.,} among those who qualify), individuals should have equal chances of a positive prediction. It aligns with rewarding merit and effort by evaluating parity among those who are  truly qualified for a positive outcome.

\begin{definition}\textbf{\emph{Sufficiency:}} For a binary classifier $f_t$, sufficiency  corresponds to $Y \indep A|\hat{Y}$, or equivalently,
\begin{equation}\tag{SUF}\label{SUF}
\begin{aligned}
\p[Y = 1 \ |\hat{Y}=1, A = 0] = \p[Y= 1 \ |\hat{Y}=1, A = 1]  \\
\p[Y = 1 \ |\hat{Y}=0, A = 0] = \p[Y= 1 \ |\hat{Y}=0, A = 1]
\end{aligned}
\end{equation}
where $\p[Y = 1 \ |\hat{Y}=1, A = g]$ and 
$\p[Y = 1 \ |\hat{Y}=0, A = g]$ correspond to the positive predictive value (PPV) and false omission rate (FOR) for group $g$, respectively.
\end{definition}

 Sufficiency, also known as conditional use accuracy (\cite{berk2021fairness}), balances the precision of prediction across groups. It requires  both the positive predictive value and the false omission rate to be equal across groups.\footnote{A relaxation of this criterion is predictive parity \citep{chouldechova2017fair}, which requires equality  of  positive predictive values across groups.}

\subsection{Explainability with the Shapley value}\label{section-explainability}

The  cooperative game theory literature has paved the way for measuring  disparities in outcomes across groups, with a special focus on income inequality decomposition and inequality-reducing policies, see among others, \citet{chantreuil2011,Chantreuil2013,Chantreuil2019,takeng2023, Fourrey2023}. Within  this literature, the Shapley value \citep{shapley1953}, a concept from cooperative game theory, serves as a classic tool to quantify how different factors (\textit{e.g.,} education, gender) contribute to overall  income inequality. 

The computer science literature has also adopted  the Shapley value as a standard tool for XAI. In particular, it is used to explain model predictions by attributing to each feature its contribution  to the outcome. These attribution methods\footnote{ Attribution methods  provide either global or local explanations: global explainability focuses on understanding the model behavior over the entire dataset, whereas local explainability aims to understand individual predictions.} quantify feature  importance, \textit{i.e,} how they drive the predictions  \citep{lundberg2017,ancona2019,linardatos2021}. Within this framework, \citet{pelegrina2024shapley} use the Shapley value to attribute each feature's contribution
to  predictive performance, 
such as the  true positive rate (TPR) and the false positive rate (FPR). In what follows, we use Shapley and ESL values more generally for explainability (ESL values will be introduced later in Section \ref{ULES} to generalize our approach).

In cooperative game theory, the Shapley value is defined as an attribution method (or allocation rule) that allocates to each player a unique share of the total payoff, based on their marginal contributions  across  all possible coalitions. A game is characterized by the pair \( (\mathcal{N}, v) \), where $\mathcal{N}=\{1,\ldots,N\}$ is  the set of $N$ players and 
\( v: 2^{\mathcal{N}} \rightarrow \mathbb{R}\) is the characteristic function  assigning a worth to each coalition \( \s_{\mathcal N} \subseteq \mathcal{N} \). 
We start the presentation of the Shapley value with the usual context of explainability, where players are features.\footnote{In Section 3, we adapt the definition, taking the levels of sensitive feature, as players.} Each feature $X_k$ of the the random input vector $\mathbf{X} := [X_1, \ldots, X_N] \in \mathbb{R}^N$  is  a ``player", where the total number of features is \( N = |\mathcal{N}| \). A coalition refers to a subset of features.  For any subset \( \mathcal{S}_{\mathcal{N}} \subseteq \mathcal{N} \), the vector \( \mathbf{X}(\mathcal{S}_{\mathcal{N}}) = (X_i, \dots, X_s) \) contains only  the features  in \( \mathcal{S}_{\mathcal{N}} \), while  features  in \( \mathcal{N} \setminus \mathcal{S}_{\mathcal{N}} \) are dropped. The total payoff $v(\mathcal{N})$ is the metric to be explained using the full set of features. The goal is to allocate  $v(\mathcal{N})$ across features based on their contributions. In  cooperative game theory, a common convention is to set $v(\emptyset)=0$, when the coalition is empty. However, in explainability settings, having an empty feature set does not directly translate to  zero payoffs. Instead  \( v(\emptyset) \) is set to a baseline. For  binary classification, a natural  baseline is the prior class probability $ v(\emptyset) = \mathbb{P}(Y = 1)$, ensuring that the baseline used in the Shapley value reflects the class distribution \citep{pelegrina2024shapley}.

To evaluate the marginal contribution of  feature \( X_k \) to a subset \( \mathcal{S}_{\mathcal{N}} \), we consider the augmented subset \( \mathcal{S}_{\mathcal{N}} \cup \{k\} \) and  the corresponding feature vector  \( \mathbf{X}(\mathcal{S}_{\mathcal{N}} \cup \{k\}) = (X_i, \dots, X_s, X_k) \).  In order to quantify the marginal impact of feature \( X_k \), the marginal contribution of \( X_k \) to a subset \( \s_{\mathcal N} \) is expressed as $v(X(\s_{\mathcal N} \cup \{k\})) - v(X(\s_{\mathcal N}))$. An attribution method is said to be \textit{marginalist} if it assigns to each feature a payoff that depends on its average marginal contributions referred as feature contribution. All feature contributions are grouped into the payoff vector \( \varphi(\mathbf{X}, v) \): 
$$ 
\varphi(\mathbf{X}, v) = [\varphi_1(\mathbf{X}, v), \dots, \varphi_N(\mathbf{X}, v)] \in \mathbb{R}^N 
$$ 
The Shapley value  for feature $k \in \mathcal{N}$ is  defined by,
\begin{equation}\label{feature-explain}
\varphi^{Sh}_k(\mathbf{X}, v) := \sum_{\s_{\mathcal{N}} \subseteq \mathcal{N} \setminus \{k\}} P(\s_{\mathcal{N}}) \Big(v(X(\s_{\mathcal{N}}\cup \{k\})) - v(X(\s_{\mathcal{N}})) \Big)
\end{equation}
 where $P(\s_{\mathcal{N}}):=(N-S-1)!S! / N!$, with $S := |\s_{\mathcal{N}}|$. Each $\varphi_k^{Sh}$ represents feature  \( X_k \) contribution for the metric being explained. Note that  $\varphi_{k}^{Sh}(X, v)$ can be either positive or negative, depending on the choice of the characteristic function $v$. The Shapley value satisfies four axioms \footnote{The standard Shapley axioms are: \textit{efficiency}, if $\sum_{k \in N} \varphi_{k}^{Sh}(\mathbf{X}, v) = v(X(\mathcal{N}))- v(X(\emptyset))$ ;  \textit{symmetry}, if $\varphi_{\pi (k)}^{Sh}(\mathbf{X},\pi v) = \varphi_{k}^{Sh}(\mathbf{X}, v)$ for every permutation $\pi$ on $\mathcal{N}$ (for all $k\in N$) ; \textit{additivity}, if $\varphi_{k}^{Sh}(\mathbf{X}, v + w) 
    = \varphi_{k}^{Sh}(\mathbf{X},v) + \varphi_{k}^{Sh}(\mathbf{X},w)$ 
    for all games $(\mathcal{N},v)$ (for all $k\in N$) ;  \textit{dummy player property}, if $\varphi_{k}^{Sh}(\mathbf{X},v) = v(X(\emptyset))$, 
    such that the player $i$ is a dummy in the game $( \mathcal{N},v )$, 
    \textit{i.e.},\ $v(X(\s_{\mathcal{N}}\cup \{i\})) - v(X(\s_{\mathcal{N}})) = v(X(\emptyset))$ for all $\s_{\mathcal{N}} \subseteq \mathcal{N} \setminus \{i\}$.} \citep{shapley1953} 
and  have been widely invoked to justify the use of Shapley values in explainability \citep{strumbelj2010efficient}.

Our contribution is the following: we adapt the Shapley value to contexts where players are not only features, but can also be levels of the sensitive feature. This approach is developed in the following section.

\section{Shapley Value for Fairness and Explainability}\label{FER}

In this section, we introduce a Shapley-based framework that connects  group fairness and explainability. We first present our main insight on group-Shapley which establishes a link between fairness criteria and the Shapley value. We then extend this result by introducing a decomposition that both tests whether a given fairness criterion holds and attributes any detected disparities to specific features. This yields a direct equivalence between group fairness criteria and explainability, thereby  unifying two concepts that are often studied in isolation.

\subsection{Group level explainability}\label{section3}

Let $A$ be a sensitive feature (\textit{e.g.}, gender), with two levels, called groups.\footnote{The setup can be adapted to a sensitive feature  with more than two levels.} Denote the set of groups by $\mathcal{A} = \{1,2\}$, where each group $g \in \mathcal{A}$ corresponds to the subset of observations associated with sensitive feature level  (\textit{e.g.}, men, women). Let $ \s_{\mathcal{A}}\in \{ \emptyset, \{1\},\{2\},\{1,2\}\}$ denote a coalition of groups. With a slight language abuse, we call coalition both the sets of levels of the sensitive feature, and the subset of data associated to these levels.
For a feature set $\mathcal{N}$ and a coalition $ \s_{\mathcal{A}}$, the feature vector restricted to $\mathcal{S}_{\mathcal{A}}$, 
is denoted  $ X \big|_{\mathcal{S}_{\mathcal{A}}}(\mathcal{N})$. An empty coalition $\emptyset$ represents the situation in which no group (hence, no data) is included. This is similar to explainability, where the empty coalition corresponds to the absence of any feature.

Let $m$ denote a classification \emph{metric} such as selection rate (SR), true positive rate (TPR), false positive rate (FPR), positive predictive value (PPV) and false omission rate (FOR). Let $v^{(m)}:2^{\mathcal{A}} \rightarrow \mathbb R$ be the characteristic function associated with $m$. 
Since the objective is to explain the contribution of each group to $v^{(m)}$ evaluated on the full dataset $X \big|_{\mathcal{A}} (\mathcal{N})$, we define $v^{(m)}(X \big|_{\mathcal{A}} (\mathcal{N}))$ as the  relative deviation of the classification metric from the baseline defined by a random classifier. The following  indicator function
is introduced:
\[
\mathds{1}_{|\s_{\mathcal{A}}|}  =
\begin{cases}
1 & \text{if } |\s_{\mathcal{A}}| =0  \\
0 & \text{otherwise} .
\end{cases}
\]
This ensures that $v^{(m)}(X \big|_{\{\emptyset\}} (\mathcal{N}))=0$ remains \emph{null} as in standard cooperative games.\footnote{ Since the empty coalition corresponds to the baseline case of no information, setting its value to zero ensures that the contribution sum to the total payoff, \textit{i.e.}, the efficiency axiom holds.
}

Consistent with the case of feature explainability, the baseline value of a random binary classifier $\mathbb{P}(\hat{Y} = 1 \mid A \in \emptyset)$ is set as follows:  if  the two outcomes are equally likely, the baseline is set to 0.5.  Otherwise, it is set to the overall proportion of the positive class $\mathbb{P}(Y = 1)$. 

Formally, the classifier trained on the combined dataset (for both men and women) is evaluated using the characteristic function  $v^{(SR)}$ defined as the relative selection rate of the classifier
\begin{align*}
 v^{(SR)}(X \big|_{\s_{\mathcal{A}}} (\mathcal{N})) &= 
\frac{\mathbb{P}(\hat{Y} = 1 \mid A \in \s_{\mathcal{A}})}{\mathbb{P}(\hat{Y} = 1 \mid A \in \emptyset)}(1-\mathds{1}_{|\s_{\mathcal{A}}|}), \ \text{with } \mathbb{P}(\hat{Y} = 1 \mid A \in \emptyset) > 0
\end{align*}

Hence, $v^{(SR)}(X \big|_{\s_{\mathcal{A}}}  (\mathcal{N}))$   measures the allocation of positive outcomes  when including $\s_{\mathcal{A}}$ relative to the baseline.  A higher selection rate is viewed as either a favorable (\textit{e.g.}, loan approval) or undesirable outcome (\textit{e.g.}, recidivism) depending on the context. 
Notice that having  $v^{(SR)}(X \big|_{\s_{\mathcal{A}}}  (\mathcal{N}))>1$,  indicates that 
the classifier assigns more positive predictions than the  baseline.  Conversely,  $v^{(SR)}(X \big|_{\s_{\mathcal{A}}}  (\mathcal{N}))<1$  suggests that  it assigns fewer positive predictions.

Similarly, let $v^{(TPR)},v^{(FPR)},v^{(PPV)}$ and $v^{(FOR)}$  be the relative performance rate of the classifier:
\begin{align*}
 v^{(TPR)}(X \big|_{\s_{\mathcal{A}}} (\mathcal{N})) &=  
\frac{\mathbb{P}(\hat{Y} = 1 \mid Y=1, A \in \s_{\mathcal{A}})}{\mathbb{P}(\hat{Y} = 1 \mid Y=1, A \in \emptyset)}(1-\mathds{1}_{|\s_{\mathcal{A}}|}), \ \text{with }  \mathbb{P}(\hat{Y} = 1 \mid Y=1, A \in \emptyset) > 0  \\
 v^{(FPR)}(X \big|_{\s_{\mathcal{A}}} (\mathcal{N})) &=  
\frac{\mathbb{P}(\hat{Y} = 1 \mid Y=0, A \in \s_{\mathcal{A}})}{\mathbb{P}(\hat{Y} = 1 \mid Y=0, A \in \emptyset)}(1-\mathds{1}_{|\s_{\mathcal{A}}|}), \ \text{with } \mathbb{P}(\hat{Y} = 1 \mid Y=0, A \in \emptyset) > 0  \\
  v^{(PPV)}(X \big|_{\s_{\mathcal{A}}} (\mathcal{N})) 
&= 
\frac{\mathbb{P}(Y = 1 \mid \hat{Y}=1, A \in \s_{\mathcal{A}})}{\mathbb{P}(Y = 1 \mid \hat{Y}=1, A \in \emptyset)}(1-\mathds{1}_{|\s_{\mathcal{A}}|}), \ \text{with } \mathbb{P}(Y = 1 \mid \hat{Y}=1, A \in \emptyset) > 0 \\
  v^{(FOR)}(X \big|_{\s_{\mathcal{A}}} (\mathcal{N})) 
&= 
\frac{\mathbb{P}(Y = 1 \mid \hat{Y}=0, A \in \s_{\mathcal{A}})}{\mathbb{P}(Y = 1 \mid \hat{Y}=0, A \in \emptyset)}(1-\mathds{1}_{|\s_{\mathcal{A}}|}), \ \text{with } \mathbb{P}(Y = 1 \mid \hat{Y}=0, A \in \emptyset) >0
\end{align*}
In this case, $v^{(TPR)}(X \big|_{\s_{\mathcal{A}}}  (\mathcal{N}))$  measures 
the ability of the classifier to detect true positives, while $v^{(FPR)}(X \big|_{\s_{\mathcal{A}}}  (\mathcal{N}))$ captures false positives predictions. When $v^{(TPR)}(X \big|_{\s_{\mathcal{A}}} (\mathcal{N}))>1$, the classifier ability to capture true positives improves when including $\s_{\mathcal{A}}$ relative to the baseline, while $v^{(TPR)}(X\big|_{\s_{\mathcal{A}}}  (\mathcal{N}))<1$ reflects how  the classifier is missing true positives under the same group. Conversely, $v^{(FPR)}(X\big|_{\s_{\mathcal{A}}}  (\mathcal{N}))>1$ indicates a higher false positives when including   $\s_{\mathcal{A}}$ relative to the baseline,  which is undesirable and can be costly (\textit{e.g.}, in the context of judicial or medical decisions), while $v^{(FPR)}(X \big|_{\s_{\mathcal{A}}}  (\mathcal{N}))<1$ is often desirable.  On the other hand, $v^{(PPV)}(X \big|_{\s_{\mathcal{A}}}  (\mathcal{N}))$ and $v^{(FOR)}(X \big|_{\s_{\mathcal{A}}}  (\mathcal{N}))$ capture the change in  PPV  and  FOR, respectively. As such, when $v^{(PPV)}(X \big|_{\s_{\mathcal{A}}}  (\mathcal{N}))>1$, precision is high when including $\s_{\mathcal{A}}$ relative to the baseline, while it is low when $v^{(PPV)}(X \big|_{\s_{\mathcal{A}}}  (\mathcal{N}))<1$. However, having $v^{(FOR)}(X \big|_{\s_{\mathcal{A}}}  (\mathcal{N}))>1$ reflects how the classifier under-performs in predicting true negatives when including $\s_{\mathcal{A}}$ relative to the baseline, while it is doing better in predicting them when $v^{(FOR)}(X \big|_{\s_{\mathcal{A}}}  (\mathcal{N}))<1$. Finally,  having $v^{(m)}(X \big|_{\s_{\mathcal{A}}}(\mathcal{N}))=1$ refers to the situation where the classifier does not differ from a random classifier.

Following the literature on income inequality decomposition \citep{Chantreuil2013,Chantreuil2019},  
the Shapley value is used to decompose the overall ratio $v^{(m)}(X \big|_{\mathcal{A}} (\mathcal{N}))$ in order to  quantify each group's average marginal contribution to it. According to the chosen definition of  $v^{(m)}(X\big|_{\mathcal{A}}  (\mathcal{N}))$, the Shapley value associated with group  $g\in \mathcal{A}$, \textit{i.e.}, the group-Shapley, is defined as:

\begin{equation}\label{subpopulation-explain}
\varphi^{Sh}_{g}(\mathbf{X}, v^{(m)}):=\sum_{\s_{\mathcal{A}} \subseteq \mathcal{A} \setminus \{g\}} P(\s_{\mathcal{A}}) \Big(v^{(m)}(X \big|_{\s_{\mathcal{A}} \cup \{g\}} (\mathcal{N})
) - v^{(m)}(X \big|_{\s_{\mathcal{A}}} (\mathcal{N})) \Big)
\end{equation}
where $P(\s_{\mathcal{A}}):=(a-s-1 )!s! / a!$, with 
 $a:=|\mathcal{A}|=2$ and $s := |\s_{\mathcal{A}}|$. 

Applied to the case of a sensitive feature that can take two group levels $g\in\{1,2\}$, the equation simplified to:

\begin{align*}
\varphi^{Sh}_{g=1}(\mathbf{X},v^{(m)})
&=
\frac{1}{2}\Big[
v^{(m)}(X|_{\mathcal{A}}(\mathcal{N}))
+
v^{(m)}(X|_{\{g=1\}}(\mathcal{N}))
-
v^{(m)}(X|_{\{g=2\}}(\mathcal{N}))
\Big]  \\
\varphi^{Sh}_{g=2}(\mathbf{X},v^{(m)})
&=
\frac{1}{2}\Big[
v^{(m)}(X|_{\mathcal{A}}(\mathcal{N}))
+
v^{(m)}(X|_{\{g=2\}}(\mathcal{N}))
-
v^{(m)}(X|_{\{g=1\}}(\mathcal{N}))
\Big]. 
\end{align*}

In what follows, we refer to $\varphi^{Sh}_{g}$ as the \textit{group-Shapley}. As in the case of feature explainability (where players are features), when players are groups, our group-Shapley $\varphi^{Sh}_{g}$ satisfies the efficiency axiom, \textit{i.e.,} $\sum_{g=1}^{a}\varphi^{Sh}_{g}(\mathbf{X}, v^{(m)})=v^{(m)}(X\big|_{\mathcal{A}}(\mathcal{N}))$. 


\subsection{Equivalences: First-stage Shapley}

We first explain how each group impacts the behavior of the classifier relative to the baseline. In the  binary case, let us take $\mathcal{A} := \{1,2\}$, where $\{1\}$ corresponds to  the subset of data associated with men and $\{2\}$ to the subset associated with women. This leads to a situation with four cases (like in standard 2-player games), where each case is obtained by filtering the rows of a dataset according to the sensitive feature:

$X\big|_{\s_{\mathcal{A}}=\emptyset} (\mathcal{N})$:  All the groups (\textit{i.e.}, all the observations) are excluded. Since there is no pattern to learn from the data, the output of a random classifier is used. 

$ X \big|_{\s_{\mathcal{A}}=\{1\}} (\mathcal{N})$: 
The classifier is evaluated on the men group (\textit{i.e.,} the subset of data associated to men). 

$  X \big|_{\s_{\mathcal{A}}=\{2\}} (\mathcal{N})$: 
The classifier is evaluated on the women group (\textit{i.e.,}  the subset of data associated to women).

$X \big|_{\mathcal{A}} (\mathcal{N})$:
The classifier is evaluated on the union of both groups.\\
Then, the Shapley value is used to understand how each group drives the classification metric with respect to the baseline. In the following remark, group explainability via the Shapley value is shown to be equivalent to group fairness. 

\begin{remark}[Group Fairness $\Leftrightarrow$ Shapley]\label{equiv1.3}
Let $f_t$ be a binary  classifier and $\varphi^{Sh}$ the Shapley value. Equivalences:
\newline \emph{(i)}  $f_t$  respects \emph{(\ref{IND})}  
\ $\Leftrightarrow$ \ $\varphi_{g=1}^{Sh}(\mathbf{X}, v^{(SR)}) = \varphi_{g=2}^{Sh}(\mathbf{X}, v^{(SR)})$
\newline \emph{(ii)} $f_t$ respects \emph{(\ref{SEP})}  $\Leftrightarrow \left\{\begin{aligned}
\varphi_{g=1}^{Sh}\!\bigl(\mathbf{X},\allowbreak v^{(TPR)}\bigr)
&= \varphi_{g=2}^{Sh}\!\bigl(\mathbf{X},\allowbreak v^{(TPR)}\bigr)
\\[-2pt]
\varphi_{g=1}^{Sh}\!\bigl(\mathbf{X},\allowbreak v^{(FPR)}\bigr)
&= \varphi_{g=2}^{Sh}\!\bigl(\mathbf{X},\allowbreak v^{(FPR)}\bigr)
\end{aligned}\right.$
\newline \emph{(iii)} $f_t$ respects \emph{(\ref{SUF})}  $\Leftrightarrow$ 
$
\left\{
\begin{aligned}
\varphi_{g=1}^{Sh}\!\bigl(\mathbf{X},\allowbreak v^{(PPV)}\bigr)
&= \varphi_{g=2}^{Sh}\!\bigl(\mathbf{X},\allowbreak v^{(PPV)}\bigr)
\\[-2pt]
\varphi_{g=1}^{Sh}\!\bigl(\mathbf{X},\allowbreak v^{(FOR)}\bigr)
&= \varphi_{g=2}^{Sh}\!\bigl(\mathbf{X},\allowbreak v^{(FOR)}\bigr).
\end{aligned}\right.
$

\end{remark}

\begin{proof}
    See the Appendix \ref{Appendix-A}. 
\end{proof}


In the binary case considered here, interpreting fairness criteria in terms of Shapley values, allows us to leverage on the decomposability of the latter, in order to compute the contribution of each feature to unfairness.

\subsection{Equivalences: Two-stage Shapley}\label{subss3}

Remark \ref{equiv1.3} focuses on the  identification of unfairness, we now build on the link between unfairness and Shapley values, combined with the literature on feature \emph{contributions}, to explain unfairness. In particular, we extend our group-Shapley with a two-stage Shapley decomposition, which allows us to evaluate the contribution of each feature to group unfairness.

The two-stage Shapley value (also called nested Shapley value) was introduced by \citet{Chantreuil1999,Chantreuil2013} to decompose income inequality indices. We adapt this nested decomposition to our setting in order to decompose each group-Shapley into  feature contributions. Specifically, we denote  $\mathcal{C}^k_{g}(v^{(m)})$ the contribution of feature 
$X_k$
 to the group-Shapley  of  $g$ in $\mathcal{A}$.

For all \( \mathcal{S}_{\mathcal{N}} \subseteq \mathcal{N} \), let $X \big|_{\s_{\mathcal{A}}} (\mathcal{S_{\mathcal{N}}})$  be the vector of features in  $\mathcal{S_{\mathcal{N}}}$ restricted to $\s_{\mathcal{A}}$.
Formally, the two-stage Shapley is obtained as follows:
\begin{itemize}
    \item In the  first stage, group-Shapley $\varphi_{g}^{Sh}(\mathbf{X}, v^{(m)})$ is derived (Eq.\ref{subpopulation-explain}).
   \item In the second stage, for every nonempty feature set
\(T\subseteq\mathcal S\) and every group \(g\in\mathcal A\), the
characteristic function is
\(\varphi_g^{\mathrm{Sh}}(\mathbf X(R),v^{(m)})\), \(R\subseteq T\), with \(\varphi_g^{\mathrm{Sh}}(\emptyset,v^{(m)})=0\).
The metric feature game is
\(v^{(m)}(\mathbf X|_{\mathcal A}(R))\) for nonempty \(R\),
and $v^{(m)}(\mathbf X|_{\mathcal A}(\emptyset))=0$.
These zero conventions apply to the second-stage feature games,
independently of the random-classifier baselines used to normalize
the classification metrics. The Shapley value is applied to this characteristic function to
measure the contribution of each feature \(X_k\), \(k\in T\), to the
group-Shapley value. More precisely,
\[
\mathcal C_{g,T}^k\!\bigl(v^{(m)}\bigr)
:=
\varphi_{k,T}^{\mathrm{Sh}}
\circ
\varphi_g^{\mathrm{Sh}}
\bigl(\mathbf X(T),v^{(m)}\bigr),
\]
where the composition operator $\circ$ indicates that feature Shapley values are implemented over group-Shapley. By efficiency of the Shapley value, these feature contributions add
up to the group-Shapley value: $\varphi_g^{\mathrm{Sh}}
\bigl(\mathbf X(T),v^{(m)}\bigr)
=
\sum_{k\in T}
\mathcal C_{g,T}^k\!\bigl(v^{(m)}\bigr)$. For the grand feature game \(T=\mathcal S\) with $T \neq \emptyset$, we use the simplified
notation $\mathcal C_g^k\!\bigl(v^{(m)}\bigr)
:=
\mathcal C_{g,\mathcal S}^k\!\bigl(v^{(m)}\bigr)$. 
\end{itemize}

\begin{proposition}[Two-stage Shapley]
\label{equiv-2}
Let \(f_t\) be a binary classifier. Equivalences:
\begin{enumerate}
\item[\emph{(i)}]
The classifier \(f_t\) respects \eqref{IND} on \(\mathbf X(T)\) for
every nonempty \(T\subseteq\mathcal S\) if and only if
\[
\mathcal C_{g,T}^k\!\bigl(v^{(SR)}\bigr)
=
\frac12\,
\varphi_{k,T}^{\mathrm{Sh}}
\!\bigl(\mathbf X(T),v^{(SR)}\bigr), \ \forall T\subseteq\mathcal S, \  \forall g\in\mathcal A, \ \forall k\in T\neq \emptyset.
\]
\item[\emph{(ii)}]
The classifier \(f_t\) respects \eqref{SEP} on \(\mathbf X(T)\) for
every nonempty \(T\subseteq\mathcal S\) if and only if
\[
\begin{cases}
\displaystyle
\mathcal C_{g,T}^k\!\bigl(v^{(TPR)}\bigr)
=
\frac12\,
\varphi_{k,T}^{\mathrm{Sh}}
\!\bigl(\mathbf X(T),v^{(TPR)}\bigr) , \ \forall T\subseteq\mathcal S, \  \forall g\in\mathcal A, \ \forall k\in T\neq \emptyset. \\
\displaystyle
\mathcal C_{g,T}^k\!\bigl(v^{(FPR)}\bigr)
=
\frac12\,
\varphi_{k,T}^{\mathrm{Sh}}
\!\bigl(\mathbf X(T),v^{(FPR)}\bigr), \ \forall T\subseteq\mathcal S, \  \forall g\in\mathcal A, \ \forall k\in T\neq \emptyset.
\end{cases}
\]
\item[\emph{(iii)}]
The classifier \(f_t\) respects \eqref{SUF} on \(\mathbf X(T)\) for
every nonempty \(T\subseteq\mathcal S\) if and only if
\[
\begin{cases}
\displaystyle
\mathcal C_{g,T}^k\!\bigl(v^{(PPV)}\bigr)
=
\frac12\,
\varphi_{k,T}^{\mathrm{Sh}}
\!\bigl(\mathbf X(T),v^{(PPV)}\bigr) , \ \forall T\subseteq\mathcal S, \  \forall g\in\mathcal A, \ \forall k\in T\neq \emptyset. \\
\displaystyle
\mathcal C_{g,T}^k\!\bigl(v^{(FOR)}\bigr)
=
\frac12\,
\varphi_{k,T}^{\mathrm{Sh}}
\!\bigl(\mathbf X(T),v^{(FOR)}\bigr) , \ \forall T\subseteq\mathcal S, \  \forall g\in\mathcal A, \ \forall k\in T\neq \emptyset.
\end{cases}
\]
\end{enumerate}

\end{proposition}

\begin{proof}
 See the Appendix  \ref{Appendix-A}.
\end{proof}

Proposition \ref{equiv-2} focuses on a stringent fairness conditions, where fairness holds for any  combination of features. For a classifier to satisfy group fairness on every \(T\subseteq\mathcal S\), each
feature's contribution must be equalized across groups and equal to
one half of its Shapley contribution to the corresponding metric.

Conversely, the equality of these contributions in every restricted
feature game implies the corresponding group fairness criterion.However, fairness of the full classifier alone does not imply equality of feature contributions across groups: positive and negative feature-level disparities can cancel, yielding a zero overall fairness gap.
When the group-Shapley values differ across groups, decomposing each
group-Shapley value into feature contributions allows us to identify
the features driving this disparity. Concretely, in the full feature
game, the difference between the group-Shapley values can be written as
\[
\Delta\varphi_g^{\mathrm{Sh}}
\bigl(\mathbf X,v^{(m)}\bigr)
:=
\varphi_{g=1}^{\mathrm{Sh}}
\bigl(\mathbf X,v^{(m)}\bigr)
-
\varphi_{g=2}^{\mathrm{Sh}}
\bigl(\mathbf X,v^{(m)}\bigr)
=
\sum_{k\in\mathcal S}
\left[
\mathcal C_{g=1}^k\!\bigl(v^{(m)}\bigr)
-
\mathcal C_{g=2}^k\!\bigl(v^{(m)}\bigr)
\right].
\]
Thus, the difference between the two contributions of each feature
identifies that feature's contribution to the overall group disparity.

\begin{remark}
By the additivity of the Shapley value, decomposing  each $\varphi_{g}^{Sh}(\mathbf{X}, v^{(m)})$  by features is the same as decomposing the difference $\Delta \varphi_{g}^{Sh}(\mathbf{X}, v^{(m)})$:
\begin{align*}
 \varphi_{k}^{Sh} \circ \Big(\Delta \varphi^{Sh}_{g}(\mathbf{X}, v^{(m)})\Big)= \varphi_{k}^{Sh} \circ \varphi_{g=1}^{Sh}(\mathbf{X}, v^{(m)})-\varphi_{k}^{Sh} \circ\varphi_{g=2}^{Sh}(\mathbf{X}, v^{(m)})
\end{align*}
Consequently, the discrepancy in group-Shapley can be explained through the difference in feature contributions across groups:
\begin{align*}
 \varphi_{g=1}^{Sh}(\mathbf{X}, v^{(m)})-\varphi_{g=2}^{Sh}(\mathbf{X}, v^{(m)}) =   (\mathcal{C}^1_{1}(v^{(m)})-\mathcal{C}^1_{2}(v^{(m)}))+\cdots+ (\mathcal{C}^N_{1}(v^{(m)})-\mathcal{C}^N_{2}(v^{(m)}))
\end{align*}
\end{remark}

Figure \ref{fig:fig2} illustrates the process of the two-stage Shapley. In the first stage, the chosen performance metric evaluated on the full dataset  is decomposed by group, yielding group-Shapley. In the second stage, each group-Shapley is further decomposed over features, yielding feature contributions within each group. Comparing these feature contributions across groups reveals which features drive the group unfairness.

\begin{figure}[H]
    \centering
    \includegraphics[width=1\linewidth]{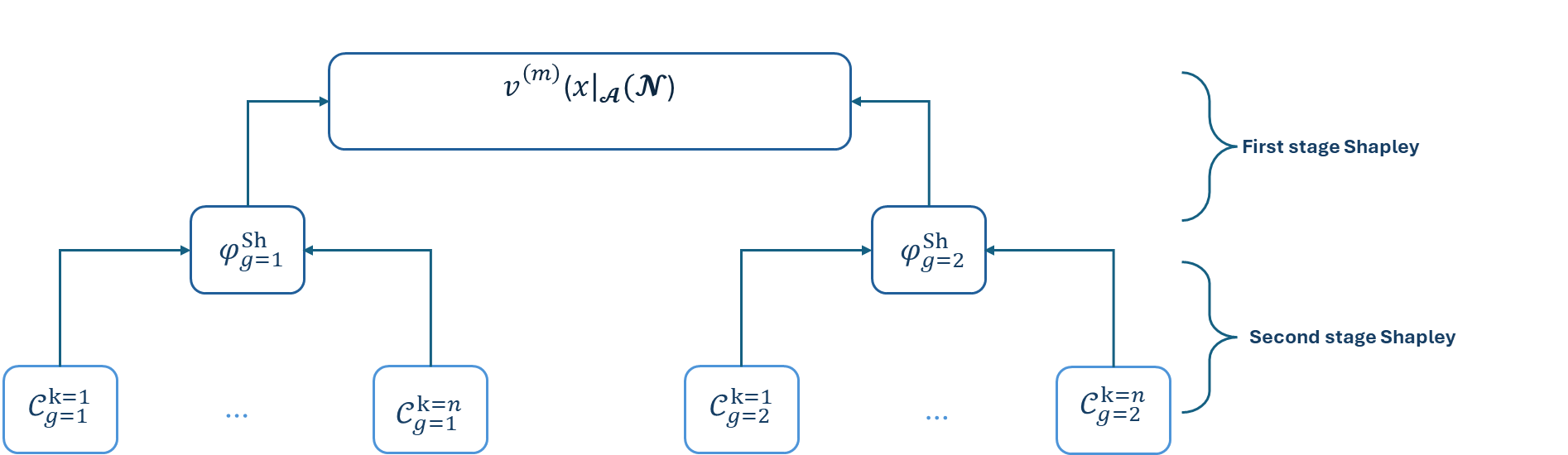}
      \caption{Two-stage Shapley attribution method}
  \label{fig:fig2}
\end{figure}

To check the robustness of the proposed two-stage Shapley, majority voting of group fairness along different ESL values is proposed in the next section.  

\section{Additional results: robustness, inference and implementability}\label{ULES}

We now propose two directions to improve the robustness and practical implementability of our framework. First, we extend our results to a broader family of allocation rules, the ESL values, which also share axiomatic foundations with the Shapley value (efficiency, symmetry, and linearity) but differ in their allocation procedures. This extension enables faster computation and more robust conclusions through majority voting across ESL allocation rules. Second, we develop statistical inference procedures for group unfairness. We derive asymptotically normal estimators for both group-ESL and second-stage feature contributions, enabling formal hypothesis testing of fairness criteria and of the drivers of unfairness.

\subsection{Unfairness explanation under ESL}


In the literature on cooperative game theory, a value satisfying efficiency, symmetry and linearity axioms  is called Efficient-Symmetric-Linear value (\textit{i.e.,} ESL value). The family of ESL values includes allocation rules such as, the Shapley value \citep{shapley1953}, the Solidarity value \citep{Nowak1994}, the Consensus value \citep{ju2007consensus}, the Equal surplus value \citep{Driessen1991} and the Least Square Prenucleolus value \citep{ruiz1996least}. ESL values have been extensively characterized in the literature \citep{hernandez2008solutions,nembua2012linear,,radzik2013family,nembua2016ordinal} and generalize the Shapley value as shown in Proposition \ref{LEStheo} recalled below.
\begin{proposition}\label{LEStheo}\textbf{\emph{(\cite{radzik2013family}):}} 
A value $\varphi_{g}^{ESL}(\mathbf{X}, v^{(m)})$ is an ESL value if and only if there exists a unique sequence of \( a-1 \) real numbers \(\{b_s\}_{s=1}^{a-1}\) such that for each \( g \in \mathcal{A} \) with \( b_0 = 0 \) and \( b_{a} = 1 \):
\begin{align*}
\varphi_{g}^{ESL}(\mathbf{X}, v^{(m)}) = \ & \sum_{\s_{\mathcal{A}} \subseteq \mathcal{A} \setminus \{g\}} \frac{(a-s-1)!s!}{a!}  \times \left( b_{s+1} v^{(m)}(X \big|_{\s_{\mathcal{A}} \cup \{g\}} (\mathcal{N})) - b_s v^{(m)}(X \big|_{\s_{\mathcal{A}} } (\mathcal{N})) \right) \notag
\end{align*}
\end{proposition}

Each ESL value is described by a specific
collection of real constants \(\{b_s\}_{s=1}^{a-1}\) leading to different interpretations.  The  Solidarity value  $\varphi^{So}$ is defined by taking  $b_s = \frac{1}{s+1}, \ \text{for } s=1,\ldots, a-1$, such that levels within the same  coalition receive similar share of the marginal contribution of that coalition. The Least Squares Prenucleolus  $\varphi^{LSP}$ is defined with  $b_a = 1 \text{ and } b_s = \binom{a-1}{s} \frac{s}{2^{a-2}} \text{ for } s=1,\ldots, a-1$, and it focuses on  minimizing the variance of the value $v^{(m)}(X \big|_{\mathcal{A} } (\mathcal{N}))-\sum_k \varphi_k^{LSP}(\mathbf{X},v^{(m)})$. The Equal Surplus value $\varphi^{ES}$  is defined with $ b_1 = a - 1, \text{ and } b_s = 0 \text{ for } s=2,\ldots,a-1 $. It quantifies the contribution of group $g$ under the assumption that the total cooperation surplus is equally distributed among all groups.
The Shapley value is defined such that $b_s = 1 \text{ for } s=1,\ldots,a-1 $. Finally, the Consensus value  $\varphi^{Co}$ is defined with $b_1 = \frac{a}{2}, \text{ and } b_s = \frac{1}{2} \text{ for } s=2,\ldots, a-1 $, and refers to the arithmetic mean of the Shapley value and the Equal Surplus value.
In the sequel, Remark \ref{equiv1.3} is generalized to all ESL values, where $\varphi_{g}^{ESL}$ denotes group-ESL  contribution.

\begin{theorem}[Group Fairness $\Leftrightarrow$ ESL value]\label{LES1.3} 
Let $f_t$ be a binary  classifier and $\varphi^{ESL}$ any given ESL value such that $b_1 \neq 0$. Equivalences:
\newline \emph{(i)}  $f_t$  respects \emph{(\ref{IND})} 
\ $\Leftrightarrow$ \ $\varphi_{g=1}^{ESL}(\mathbf{X}, v^{(SR)}) = \varphi_{g=2}^{ESL}(\mathbf{X}, v^{(SR)})$
\newline \emph{(ii)} $f_t$ respects \emph{(\ref{SEP})} $\Leftrightarrow \left\{\begin{aligned}
\varphi_{g=1}^{ESL}\!\bigl(\mathbf{X},\allowbreak v^{(TPR)}\bigr)
&= \varphi_{g=2}^{ESL}\!\bigl(\mathbf{X},\allowbreak v^{(TPR)}\bigr)
\\[-2pt]
\varphi_{g=1}^{ESL}\!\bigl(\mathbf{X},\allowbreak v^{(FPR)}\bigr)
&= \varphi_{g=2}^{ESL}\!\bigl(\mathbf{X},\allowbreak v^{(FPR)}\bigr)
\end{aligned}\right.$
\newline \emph{(iii)} $f_t$ respects \emph{(\ref{SUF})} $\Leftrightarrow$ 
$
\left\{
\begin{aligned}
\varphi_{g=1}^{ESL}\!\bigl(\mathbf{X},\allowbreak v^{(PPV)}\bigr)
&= \varphi_{g=2}^{ESL}\!\bigl(\mathbf{X},\allowbreak v^{(PPV)}\bigr)
\\[-2pt]
\varphi_{g=1}^{ESL}\!\bigl(\mathbf{X},\allowbreak v^{(FOR)}\bigr)
&= \varphi_{g=2}^{ESL}\!\bigl(\mathbf{X},\allowbreak v^{(FOR)}\bigr).
\end{aligned}\right.
$
\end{theorem}
\begin{proof}
See the Appendix \ref{Appendix-A}.
\end{proof}

As in the case of the Shapley value, Theorem~\ref{LES1.3} establishes a generalization to all ESL values and shows that group fairness can be assessed by requiring equal group-ESL $\varphi_{g}^{ESL}$  across  groups. Hence,  a statistical test for group fairness  assessment is proposed in terms of (\ref{IND}), (\ref{SEP}) or (\ref{SUF}). The following test will be used in the empirical application: 
\begin{equation}\notag
\left\|
\begin{array}{l}
H_0 : \varphi_{g=1}^{ESL}(\mathbf{X}, v^{(m)}) = \varphi_{g=2}^{ESL}(\mathbf{X}, v^{(m)}) 
\\ 
H_1 : \varphi_{g=1}^{ESL}(\mathbf{X}, v^{(m)}) \neq \varphi_{g=2}^{ESL}(\mathbf{X}, v^{(m)}) 
\end{array}
\right.
\end{equation}

\begin{theorem}\label{stat11}
For a fixed classification metric $m \in \{ SR, TPR, FPR, PPV, FOR \}$, let $b_1\neq 0$ be determined by Proposition
\ref{LEStheo}, and $\beta_m>0$ the baseline to normalize $v^{(m)}$. Condition on the training sample
and assume that the evaluation observations are i.i.d. and
independent of it. Let $n_g^+$ be the denominator count of metric $m$ within group
$g\in\{1,2\}$. Assume that $n_g^+/n\to\lambda_g>0$ almost surely as the total
evaluation sample size $n\to\infty$. Consider the null hypothesis
\[
H_0:\quad
\varphi^{ESL}_{g=1}\!\bigl(\mathbf{X},v^{(m)}\bigr)
=
\varphi^{ESL}_{g=2}\!\bigl(\mathbf{X},v^{(m)}\bigr).
\]
Under $H_0$, assume that the common group rate $p$ satisfies
$0<p<1$. Then
\[
Z=
\frac{\widehat{\varphi}_{1}^{ESL}
      -\widehat{\varphi}_{2}^{ESL}}
{\sqrt{
\frac{b_1^2}{\beta_m^2}
\hat{p}_{\mathcal A}(1-\hat{p}_{\mathcal A})
\left(\frac{1}{n_1^+}+\frac{1}{n_2^+}\right)
}}
\xrightarrow{d}\mathcal N(0,1),
\]
where $\hat{p}_{\mathcal A}$ is the pooled empirical estimate of
the unscaled metric $m$. The same conclusion holds when $\beta_m$
is replaced by a consistent estimator used in both group
contributions and in the denominator.
\end{theorem}
\begin{proof}
    See Appendix \ref{Appendix-B}.
\end{proof}

Theorem~\ref{stat11} provides an asymptotic test statistic, which is a standard two-sample proportion test. In particular, the denominator $n_g^{+}$ depends on the metric $m$:  TPR is the proportion of actual positives in group $g$; FPR is the proportion of actual negatives; SR the group size; PPV the number of predicted positives; and FOR the false omission rate. In Appendix~\ref{Appendix-B} (see Proposition \ref{stat1}), we derive  the asymptotic variance and the resulting test statistic explicitly for $m=TPR$. The same steps can be reproduced for other classification metrics (\textit{e.g.}, $FPR$, $SR$, $PV$, $FOR$), yielding analogous test statistics. 


It is noteworthy that the literature points out that the Equal Surplus (ES) value is a relevant ESL value for text and image classifications \citep{Condevaux2022}. Contrary to the Shapley value and other ESL values that are characterized by an exponential time complexity, the ES value has a linear time complexity. Accordingly, testing the null hypothesis (\(H_0\)) regarding a specific group fairness criterion can be performed using a single member of the ESL family (see Section~\ref{EXAMPLE}). This computational advantage enhances the feasibility of hypothesis testing through the ES value, although testing several ESL values  is preferable to ensure greater robustness.

Let us now propose the two-stage ESL that generalizes Proposition \ref{equiv-2} (two-stage Shapley) to all ESL values.

\begin{corollary}[Two-stage ESL]
\label{equiv-41}
Let \(f_t\) be a binary classifier. Equivalences:
\begin{enumerate}
\item[\emph{(i)}]
The classifier \(f_t\) respects \eqref{IND} on \(\mathbf X(T)\) for
every nonempty \(T\subseteq\mathcal S\) if and only if
\[
\mathcal C_{g,T}^k\!\bigl(v^{(SR)}\bigr)
=
\frac12\,
\varphi_{k,T}^{\mathrm{ESL}}
\!\bigl(\mathbf X(T),v^{(SR)}\bigr),
\quad
\forall T\subseteq\mathcal S,\ 
\forall g\in\mathcal A,\ 
\forall k\in T \neq \emptyset.
\]
\item[\emph{(ii)}]
The classifier \(f_t\) respects \eqref{SEP} on \(\mathbf X(T)\) for
every nonempty \(T\subseteq\mathcal S\) if and only if
\[
\begin{cases}
\displaystyle
\mathcal C_{g,T}^k\!\bigl(v^{(TPR)}\bigr)
=
\frac12\,
\varphi_{k,T}^{\mathrm{ESL}}
\!\bigl(\mathbf X(T),v^{(TPR)}\bigr),
\quad
\forall T\subseteq\mathcal S,\ 
\forall g\in\mathcal A,\ 
\forall k\in T \neq \emptyset,
\\[6pt]
\displaystyle
\mathcal C_{g,T}^k\!\bigl(v^{(FPR)}\bigr)
=
\frac12\,
\varphi_{k,T}^{\mathrm{ESL}}
\!\bigl(\mathbf X(T),v^{(FPR)}\bigr),
\quad
\forall T\subseteq\mathcal S,\ 
\forall g\in\mathcal A,\ 
\forall k\in T \neq \emptyset.
\end{cases}
\]
\item[\emph{(iii)}]
The classifier \(f_t\) respects \eqref{SUF} on \(\mathbf X(T)\) for
every nonempty \(T\subseteq\mathcal S\) if and only if
\[
\begin{cases}
\displaystyle
\mathcal C_{g,T}^k\!\bigl(v^{(PPV)}\bigr)
=
\frac12\,
\varphi_{k,T}^{\mathrm{ESL}}
\!\bigl(\mathbf X(T),v^{(PPV)}\bigr),
\quad
\forall T\subseteq\mathcal S,\ 
\forall g\in\mathcal A,\ 
\forall k\in T\neq \emptyset,
\\[6pt]
\displaystyle
\mathcal C_{g,T}^k\!\bigl(v^{(FOR)}\bigr)
=
\frac12\,
\varphi_{k,T}^{\mathrm{ESL}}
\!\bigl(\mathbf X(T),v^{(FOR)}\bigr),
\quad
\forall T\subseteq\mathcal S,\ 
\forall g\in\mathcal A,\ 
\forall k\in T\neq \emptyset.
\end{cases}
\]
\end{enumerate}
\end{corollary}

\begin{proof}
Follows directly from Proposition \ref{equiv-2} 
\end{proof}

From Corollary \ref{equiv-41} another statistical test can be employed to assess the significance of the group unfairness,
over each feature $k \in \s \subseteq \N$. 
A majority voting scheme can be introduced to extend the robustness of explainability. Formally, for each ESL value, the following null hypothesis $H_0$ is tested to look for the contribution of feature $X_k$ to group unfairness. If at least 3 out of 5 ESL values reject the null hypothesis $H_0$, at a risk level $\alpha$, majority voting ensures that $H_0$ is rejected\footnote{Here, the risk level refers to each voter rather than to the majority vote.}. For a given $k \in \s \subseteq \N$, the test is the following,
\begin{equation}\notag
\left\|
\begin{array}{l}
H_0 : \mathcal{C}^k_{g=1}(v^{(m)}) = \mathcal{C}^k_{g=2}(v^{(m)}) \  
\\ 
H_1 : \mathcal{C}^k_{g=1}(v^{(m)}) \neq \mathcal{C}^k_{g=2}(v^{(m)}) \
\end{array}
\right.
\end{equation}

\begin{theorem}\label{stat12}
For a fixed classification metric $m \in \{ SR, TPR, FPR, PPV, FOR \}$, let the common normalization
baseline $b_1\neq 0$ be determined by Proposition
\ref{LEStheo}, and $\beta_m>0$ the baseline to normalize $v^{(m)}$. Condition on the training sample
and assume that the evaluation observations are i.i.d., and
independent of it. For any feature $k\in \s \subseteq \mathcal{N}$, consider the null hypothesis
\[
H_0:\quad \mathcal{C}_1^k(v^{(m)})=\mathcal{C}_2^k(v^{(m)}).
\]
Let
\[
\widehat V_k=
\widehat{\operatorname{Var}}\!\left(\mathcal{C}_1^k(v^{(m)})\right)
+
\widehat{\operatorname{Var}}\!\left(\mathcal{C}_2^k(v^{(m)})\right)
-
2\widehat{\operatorname{Cov}}\!\left(
\mathcal{C}_1^k(v^{(m)}),\mathcal{C}_2^k(v^{(m)})
\right).
\]
Assume that the asymptotic variance $\tau_k^2$ of
$\sqrt n(\mathcal{C}_1^k-\mathcal{C}_2^k)$ is strictly positive
and that $n\widehat V_k\xrightarrow{p}\tau_k^2$, where $n$
is the total evaluation sample size. Then, under $H_0$,
\[
Z_k=
\frac{\mathcal{C}_1^k(v^{(m)})-\mathcal{C}_2^k(v^{(m)})}
{\sqrt{\widehat V_k}}
\xrightarrow{d}\mathcal N(0,1).
\]
\end{theorem}
\begin{proof}
The special case $m=TPR$ is treated in Proposition \ref{stat2} of Appendix \ref{Appendix-B}. The same procedure can be applied to other classification metrics ($FPR$, $SR$, $PPV$, and $FOR$).
\end{proof}
For each feature $X_k$, Theorem~\ref{stat12} provides a test of whether the feature contribution differs significantly across groups. Rejecting $H_0$ indicates that the feature contributes differently to the metric $m$ across groups and therefore identifies a potential driver of group unfairness. As before, in Appendix \ref{Appendix-B} (see Proposition \ref{stat2}), we derive the asymptotic variance and the resulting test statistic for $m=TPR$.

These tests can be extended to multiple sensitive features. In some situations, applying group fairness independently to each sensitive feature 
can result in fairness at the marginal level, whereas taking groups defined by intersectional sensitive feature (\textit{e.g.}, gender together with ethnicity) may detect unfairness. This is called  ``fairness gerrymandering'' \citep{kearns2018preventing}. In this regard, we propose a multi-stage ESL procedure that accounts for multiple sensitive features with binary levels. The result is presented in Appendix \ref{Appendix-A} (Theorem \ref{theromMULTI}).

\section{Application on Census Data}\label{EXAMPLE}
We illustrate the practical implications of our results on the Census Income dataset from the UCI Machine Learning Repository \citep{census_income_20}.\footnote{All code is accessible at the following \href{https://github.com/anonymous03032026/fairles_shapley.git}{github} repository.} This dataset contains 48,842 instances with 15 demographic features, including two binary sensitive features: ethnicity and gender.  The task is to predict whether an individual's yearly income exceeds \$50.000. To achieve this, we adopt a soft voting classifier\footnote{The soft voting model is an ensemble learning method where different base models are used to predict the probability of each class. It then takes the average of the probabilities of each class. The class with the highest weighted probability is declared as the winning prediction. 
} that combines five classifiers, namely, Decision Trees, XGBoost, Linear Support Vector Machine, Logistic Regression and Random Forest\footnote{Our choice of classifiers is primarily motivated by their ability to include penalty terms  that address class imbalance in the data.}. The dataset is characterized by gender imbalance, with a men to women ratio of $\frac{32,650}{16,192}=2.016$. Additionally, only 25\% of the  individuals earn over \$50.000. 
A 70/30 train-test split is used, where a few features are used to train the classifier for illustration purposes, namely  age,   educational-num (years of education), worked hours-per-week (weekly hours worked), marital-status (married, never married or other), while gender (men or women) is treated as the sensitive attribute.

\subsection{Application of the main results}\label{A51}

We examine group fairness in terms of \textit{equal opportunity} by focusing on   differences in  group-ESL  contributions. Recall that equal opportunity is a relaxation of \eqref{SEP} and is equivalent to the following:
\begin{equation}\tag{SHAP-EOD}\label{SHAP-EOD}
\varphi_{g=1}^{ESL}(\mathbf{X}, v^{(TPR)}) = \varphi_{g=2}^{ESL}(\mathbf{X}, v^{(TPR)}) 
\end{equation}

According to  Theorem \ref{LES1.3}, the criterion \eqref{SHAP-EOD} can be assessed using ESL values. To address class imbalance, sampling weights are  adjusted for each classification model so that both classes contribute equally during training, with higher weights assigned to the underrepresented class. Consequently, it is appropriate to take a TPR of 0.5 as the outcome of a random classifier. To evaluate group fairness with respect to  \eqref{SHAP-EOD}, the group-ESL $\textbf{$\varphi_{g}^{ESL}(\mathbf{X}, v^{(TPR)})$}$ for each group are computed and reported  in Table \ref{tab:summary}.

\begin{table}[H]
\centering
\small
\renewcommand{\arraystretch}{1.2}
\setlength{\tabcolsep}{8pt}
\begin{tabular}{@{}lccc@{}}
\toprule
\textbf{ESL value} 
& \multirow{2}{*}{$v^{(TPR)}(X \big|_{\mathcal{A}} (\mathcal{N}))$} 
& \multicolumn{2}{c}{\textbf{$\varphi_{g}^{ESL}(\mathbf{X}, v^{(TPR)})$}} \\
& & (men) & (women) \\
\midrule
ES, Shapley, Consensus, LSP 
& \multirow{2}{*}{1.61}
& 1.047 (64.95\%) & 0.565 (35.05\%) \\
Solidarity 
& 
& 0.926 (57.48\%) & 0.685 (42.52\%) \\
\bottomrule
\end{tabular}
\caption{Summary of ESL values for men and women group}
\label{tab:summary}
\end{table}

From Table \ref{tab:summary}, the total relative deviation of TPR from the baseline  $v^{(TPR)}(X \big|_{\mathcal{A}} (\mathcal{N})) =1.61>1$, indicates that including the set $X \big|_{\mathcal{A}} (\mathcal{N})$ increases the TPR by 61\%. Yet, what matters for assessing group unfairness is the difference in group-ESL. According to ES, Shapley, Consensus and LSP, 64.95\% of the improvement is attributed to men, while women account  for only 35.05\%. According to the Solidarity value, the corresponding percentages are 57.48\% and 42.52\%, respectively. This difference between groups indicates that the improvement in TPR is unequally supported across groups, with men receiving higher contributions than women. This imbalance highlights group unfairness, as the classifier's predictive gains are not shared equally.

To assess the statistical significance of group unfairness between men and women, the difference in group-ESL is tested using: 

\[
\left\{
\begin{aligned}
H_0 : \varphi_{g=1}^{ESL}(\mathbf{X}, v^{(TPR)})-\varphi_{g=2}^{ESL}(\mathbf{X}, v^{(TPR)})=0 &  \\
H_1 :\varphi_{g=1}^{ESL}(\mathbf{X}, v^{(TPR)})-\varphi_{g=2}^{ESL}(\mathbf{X}, v^{(TPR)})\neq 0  & 
\end{aligned}
\right.
\]

When dealing with binary groups, all ESL values share the same coefficients ($b_1=1$ and $b_2=1$), except for the Solidary value (which uses $b_1=0.5$ and $b_2=1$). Since the $Z$-statistic (see Theorem \ref{stat1} in Appendix \ref{Appendix-A}) depends on $b_1$, it is  sufficient to test one ESL value out of five. The Equal Surplus was selected  because it has linear time complexity. The hypothesis test rejects the null hypothesis of fairness at the 5\% significance level,  indicating a significant difference in group-ESL  contribution of 0.482 (with 95\% CI:  $ [0.407,0.557]$).

To illustrate Proposition \ref{equiv-2} and Corollary \ref{LES1.3}, we apply the two-stage ESL to decompose the differences in  group-ESL. Table \ref{tab:summ234} reports the contribution of each feature (to the contribution of each group-ESL). In other words, we aim to assess group fairness in terms of equal opportunity, according to the following definition:
\begin{equation}\tag{F-SHAP-EOD}\label{F-SHAP-EOD}
\mathcal{C}^k_{g=1}(v^{(TPR)}) = \mathcal{C}^k_{g=2}(v^{(TPR)}), 
\end{equation}
and test for differences in ESL feature contributions across groups,
\[
\left\{
\begin{aligned}
H_0 : \mathcal{C}^k_{g=1}(v^{(TPR)}) - \mathcal{C}^k_{g=2}(v^{(TPR)})&=0    \\
H_1 : \mathcal{C}^k_{g=1}(v^{(TPR)})- \mathcal{C}^k_{g=2}(v^{(TPR)}) &\neq0
\end{aligned}
\right.
\]
Here, the null $H_0$ is interpreted as ``feature $k$ does not contribute to unfairness'' and $H_1$ reads as ``feature $k$ contributes to unfairness by reducing or widening the gap between groups''.

Table \ref{tab:summ234} reveals systematic differences in feature contributions across groups. For instance, using the ES value, ``Age'' contributes 0.681 to the men's group contribution and -0.137 to the women's group contribution, with a significant difference across groups  ($95\%$ CI: $[0.707, 0.928]]$). Conversely, ``Marital status'' contributes 0.231 to the men's group contribution and 0.633 to the women's group contribution, also showing a significant difference $(95\% \ \text{CI}: [ -0.598, -0.205]$). This pattern is consistent across other ESL values as well. Applying the majority voting rule, since at least 3 out of 5 ESL values reject $H_0$, for each of these features, the null hypothesis of equal feature contributions is rejected. These results indicate that group unfairness in feature contributions are not unique to any one ESL value but are observed consistently across multiple values. Hence (\ref{F-SHAP-EOD}) is violated, since ``Age'', ``Hours/Week'', and ``Marital status'' exhibit marked differences in their contributions across groups. These features therefore explain a large part of this gap. While \citet{pelegrina2024explaining}, identify  proxy features for gender, such as ``Marital status'', as the  main source of group unfairness, our results further indicate that ``Age'' and  ``Hours/Week'' also contribute significantly. This finding illustrates the necessity of the two-stage ESL  for a more granular group unfairness assessment and motivates the development of asymptotically normal estimators for the two-stage ESL to statistically test for the significant contributions at both stages.

\begin{table}[H]
\centering
\small
\begin{tabular}{@{}llcccc@{}}
\toprule
\textbf{Method} & \textbf{Feature} & $\mathcal{C}^k_{g=1}(v^{(TPR)})$ & \textbf{CI ($g=1$)} & $\mathcal{C}^k_{g=2}(v^{(TPR)})$ & \textbf{CI ($g=2$)} \\
\midrule
\multirow{4}{*}{ES} 
& Age & $ 0.681^{***}$ & $[0.615, 0.747]$ & $-0.137^{***}$ & $[-0.201, -0.072]$ \\
& Educ Num & 0.009 & $[-0.09, 0.109]$ & $0.161^{**}$ & $[0.061, 0.261]$ \\
& Hours/Week & $0.126^{***}$ & $[0.056, 0.196]$ & $-0.092^{**}$ & $[-0.161, -0.023]$ \\
& Marital status & $0.231^{***}$ & $[0.128, 0.334]$ & $0.633^{***}$ & $[0.527, 0.738]$ \\
\midrule
\multirow{4}{*}{Shapley} 
& Age & $ 0.424^{***}$ & $[0.380, 0.468]$ & $-0.006$ & $[-0.036, 0.024]$ \\
& Educ Num & $0.181^{***}$ & $[0.123,0.238]$ & $ 0.114^{***}$ & $[0.058, 0.169]$ \\
& Hours/Week & $0.257^{***}$ & $[0.219,0.296]$ & $ -0.023$ & $[-0.056, 0.009]$ \\
& Marital status & $0.185^{***}$ & $[0.110,0.260]$ & $0.481^{***}$ & $[0.398, 0.564]$ \\
\midrule
\multirow{4}{*}{Solidarity} 
& Age & $0.275^{***}$ & $[0.257, 0.292]$ & $0.143^{***}$ & $[0.128, 0.158]$ \\
& Educ Num & $0.198^{***}$ & $[0.178 ,0.218]$ & $ 0.156^{***}$ & $[0.137, 0.175]$ \\
& Hours/Week & $0.207^{***}$ & $[0.190,0.224]$ & $0.120^{***}$ & $[0.105, 0.135]$ \\
& Marital status & $0.247^{***}$ & $[0.225,0.269]$ & $0.266^{***}$ & $[0.241, 0.291]$ \\
\midrule
\multirow{4}{*}{Consensus} 
& Age & $0.552^{***}$ & $[0.505,0.599]$ &$-0.071^{***}$& $[-0.113, -0.030]$ \\
& Educ Num & $0.095^{*}$ & $[0.021,0.168]$ & $0.137^{***}$ & $[0.064, 0.210]$ \\
& Hours/Week & $0.192^{***}$ & $[0.144,0.240]$ & $-0.058^{*}$ & $[-0.103, -0.013]$ \\
& Marital status & $ 0.208^{***}$ & $[0.125,  0.291]$ & $0.557^{***} $ & $[0.468, 0.646]$ \\
\midrule
\multirow{4}{*}{LSP} 
& Age & $0.410^{***}$ & $[0.368,0.453]$ & $0.001$ & $[-0.032, 0.034]$ \\
& Educ Num & $0.188^{***}$ & $[0.131,0.244]$ & $  0.111 ^{***}$ & $[0.057, 0.166]$ \\
& Hours/Week & $ 0.269^{***}$ & $[0.229, 0.310]$ & $ -0.019$ & $[-0.054, 0.016]$ \\
& Marital status & $0.179^{***}$ & $[0.105,  0.253]$ & $0.472 ^{***}$ & $[0.390, 0.553]$ \\
\midrule
\multicolumn{6}{l}{\small{*$p < 0.05$, **$p < 0.01$, ***$p < 0.001$}} \\
\bottomrule
\end{tabular}
\caption{Feature contributions to group-ESL }
\label{tab:summ234}
\end{table}

\subsection{Computation times}

The implementability of our test is assessed by evaluating its computational cost. All computation times were obtained using the Toubkal supercomputer \citep{kissami2025toubkal}, using a single node with 56 CPU cores and no GPU resources. Computation is parallelized using multi-core CPU processing to accelerate the computation of LES values and their associated variance estimates.  The runtime of two approaches is compared: (i) our proposed two-stage ESL test based on a single plug-in variance estimator, and (ii) a nonparametric stratified bootstrap test for the two-stage ESL on the test set. Specifically, observations are resampled with replacement within each stratum defined by group and outcome label. For the bootstrap baseline, $B=1,000$ independent bootstrap replications are performed and parallelized across CPU cores. The first-stage group-ESL  and the second-stage feature contributions are recomputed, and  confidence intervals are constructed from the empirical distribution of the bootstrap replicates. 
In contrast to the plug-in approach (\textit{i.e.,} single variance estimate across ESL values), the bootstrap procedure requires separate recomputation for each ESL value, which increases the total runtime.

Table \ref{tab:contributions3333} reports feature contribution estimates with both asymptotic confidence intervals derived from the plug-in variance estimator and bootstrap confidence intervals. Across all ESL values and features, the two types of intervals are closely aligned with the same conclusions regarding the statistical significance of feature contributions.\footnote{
The discrepancy between the confidence intervals for "Educ Num" under the Equal Surplus and Consensus settings is mainly due to class imbalance: there are 501 positive for women compared with 2,862 for men. As a result, the variance of the estimated TPR for women is 5.7 time higher than the variance for men.} This supports the validity of the asymptotic normal approximation for two-stage ESL.

Table~\ref{tableeee}  reports the resulting computation times. Our two-stage ESL test runs in approximately 8 minutes for the five ESL values,  since the variance estimation is calculated only once for all of them. In contrast, the bootstrap approach requires approximately  more than 1h20 minutes for the second stage, as bootstrap resampling must be performed separately for each ESL value. These results highlight the  time efficiency of the two-stage ESL test for fairness evaluation compared to the computationally intensive bootstrap alternative.

\begin{table}[H]
\small
\centering
\begin{tabular}{@{}llcccc@{}}
\toprule
\textbf{Method} & \textbf{Feature} & $\Delta \mathcal{C}^k_{g}$ & \textbf{95\%CI} &$\mu_{bootstrap}(\Delta \mathcal{C}^k_{g})$& \textbf{95\% Bootstrap CI} \\
\midrule

\multirow{4}{*}{ES}
& Age           & $0.817^{***}$      & $[0.707,\,0.928]$ &  0.818    & $[0.784, 0.851]$ \\
& Educ Num      & $-0.152$           & $[-0.337,\,0.033]$ & -0.152    & $[-0.218,  -0.086]$ \\
& Hours/Week    & $0.218^{***}$      & $[0.100,\,0.337]$   & 0.218      & $[0.181,   0.255]$ \\
& Marital status  & $-0.402^{***}$     & $[-0.598,\,-0.205]$ & -0.400   & $[-0.476, -0.314]$ \\

\midrule
\multirow{4}{*}{Shapley}
& Age           & $0.430^{***}$      & $[0.366, 0.493]$   & 0.430  & $[0.399 ,   0.459]$ \\
& Educ Num      & $0.067$            & $[-0.035, 0.169]$ &  0.067   & $[ 0.012,    0.125 ]$ \\
& Hours/Week    & $0.281^{***}$      & $[0.217, 0.344]$ & 0.281    & $[0.251,    0.312]$ \\
& Marital status  & $-0.296^{***}$      & $[-0.439, -0.152]$ &  -0.294    & $[-0.377 ,  -0.210]$\\

\midrule
\multirow{4}{*}{Solidarity}
& Age           & $0.131^{***}$      & $[0.105,\,0.157]$ & 0.131      & $[0.121,   0.141]$ \\
& Educ Num      & $0.042^{**}$            &$[0.010,\,0.074]$ &   0.042    & $[0.028,   0.056]$ \\
& Hours/Week    & $0.087^{***}$       & $[0.061,\,0.113]$ &  0.087     & $[0.077  , 0.097]$ \\
& Marital status  & $-0.019$           & $[-0.059,\,0.021]$  & -0.019    & $[-0.039  , 0.003]$ \\

\midrule
\multirow{4}{*}{Consensus}
& Age           & $0.624^{***}$      & $[0.549,\,0.698]$  & 0.624  & $[0.596,   0.652]$ \\
& Educ Num      & $-0.043$           &$[-0.176,\,0.091]$  &   -0.043   & $[-0.103  , 0.016]$ \\
& Hours/Week    & $0.250^{***}$      & $[0.170,\,0.329]$ & 0.250    & $[0.220, 0.281]$ \\
& Marital status  & $-0.349^{***}$     & $[-0.508,\,-0.190]$ &-0.347    & $[-0.429 , -0.262]$ \\

\midrule
\multirow{4}{*}{LSP}
& Age           & $0.409^{***}$      &$[0.340,\,0.479]$  & 0.410    & $[0.379 ,  0.439]$ \\
& Educ Num      & $0.076$            & $[-0.025,\,0.178]$ &  0.076   & $[0.020, 0.133]$ \\
& Hours/Week    & $0.288^{***}$      & $[0.219,\,0.357]$&   0.288  & $[0.259,   0.319]$ \\
& Marital status  & $-0.292^{***}$     & $[-0.434,\,-0.151]$  & -0.291  & $[-0.373,  -0.206]$ \\

\bottomrule
\end{tabular}
\caption{\small Asymptotic confidence intervals and bootstrap confidence intervals}
\label{tab:contributions3333}
\end{table}


\begin{table}[H]
\centering
\small
\renewcommand{\arraystretch}{1.2}
\setlength{\tabcolsep}{8pt}
\begin{tabular}{@{}lcc@{}}
\toprule
Time complexity 
& Two-stage ESL 
& Bootstrap (two stage) \\
\midrule

ES
& \multirow{5}{*}{8 min }
& 1 h 37 min \\

Shapley
& 

& 1 h 34 min \\

Solidarity
&

& 1 h 18 min \\

Consensus
&

& 1 h 36 min \\

LSP
&

& 1 h 39 min \\

\bottomrule
\end{tabular}
\caption{Summary of time complexity for different methods}
\label{tableeee}
\end{table}

\subsection{Fairness mitigation}

The application to the Census Income dataset shows that our approach both assesses and explains unfairness present in the data. As a robustness check, we also verify that the method recognizes the absence of unfairness once a mitigation procedure is applied. Unfairness is corrected using Equalized Odds postprocessing algorithm\footnote{The algorithm directly implements the approach described in \citet{Hardt2016}.} from the AI Fairness 360 (AIF360) toolkit. The analyses presented in Section \ref{\A51} are then replicated  after fairness correction. Table \ref{tab:summary123} reports the group-ESL  contributions following mitigation.

\begin{table}[H]
\centering
\small
\renewcommand{\arraystretch}{1.2}
\setlength{\tabcolsep}{8pt}
\begin{tabular}{@{}lccc@{}}
\toprule
\textbf{ESL value} 
& \multirow{2}{*}{$v^{(TPR)}(X \big|_{\mathcal{A}} (\mathcal{N}))$} 
& \multicolumn{2}{c}{\textbf{$\varphi_{g}^{ESL}(\mathbf{X}, v^{(TPR)})$}} \\
& & (men) & (women) \\
\midrule
ES, Shapley, Consensus, LSP     
& \multirow{2}{*}{1.365} 
& 0.706 (51.72\%) 
& 0.659 (48.28\%) \\
Solidarity  
&  
& 0.694 (50.84\%) 
& 0.671 (49.16\%) \\
\bottomrule
\end{tabular}
\caption{Group-ESL  contributions (after fairness mitigation).}
\label{tab:summary123}
\end{table}

The difference in  group-ESL  contributions between women and men is close to zero across all ESL values. For the ES, Shapley, Consensus, and LSP, the difference is 0.047 (95\% CI: $[-0.028,0.122]$), indicating no statistically significant gender-based disparity. Then, a second-stage  decomposition is applied  to understand the consequence of fairness mitigation. Table \ref{tab:contributions_with_ci} summarizes the contributions of each feature to the group-ESL  contributions. For all ESL values, the differences in feature contributions are reduced. This decrease is reflected in more balanced contributions across groups, where  the remaining differences in contributions are not statistically significant at the \%5 level.

\begin{table}[H]
\small
\centering
\renewcommand{\arraystretch}{1.15}
\setlength{\tabcolsep}{6pt}
\begin{tabular}{@{}llcccc@{}}
\toprule
\textbf{Method} 
& \textbf{Feature} 
& $\mathcal{C}^k_{g=1}(v^{(TPR)})$ 
& $\mathcal{C}^k_{g=2}(v^{(TPR)})$ 
& $\Delta \mathcal{C}^k$ 
& \textbf{95\% CI} \\
\midrule
\multirow{4}{*}{ES}
& Age            & -0.099 & -0.097 & -0.002 & $[-0.100, 0.096]$ \\
& Educ Num       & 0.192  &  0.168 &  0.024 & $[-0.164, 0.212]$ \\
& Hours/Week     & -0.099 & -0.098 & -0.001 & $[-0.096, 0.094]$ \\
& Marital status &  0.712 &  0.685 &  0.027 & $[-0.150, 0.204]$ \\
\midrule
\multirow{4}{*}{Shapley}
& Age            & 0.027  & 0.009 &  0.018 & $[-0.062, 0.098]$ \\ 
& Educ Num       & 0.132  & 0.131  &  0.001 & $[-0.103, 0.105]$ \\
& Hours/Week     & 0.008  & -0.000 &  0.008 & $[-0.068, 0.084]$ \\
& Marital status & 0.539  & 0.520  &  0.019 & $[-0.099, 0.137]$ \\
\midrule
\multirow{4}{*}{Solidarity}
& Age            & 0.109& 0.103 &  0.006 & $[-0.023, 0.035]$ \\
& Educ Num       & 0.163 & 0.158 & 0.005 & $[-0.028, 0.038]$ \\
& Hours/Week     & 0.105 & 0.100 & 0.005  & $[-0.024, 0.034]$ \\
& Marital status & 0.317 & 0.309 & 0.008 & $[-0.027, 0.043]$ \\
\midrule
\multirow{4}{*}{Consensus}
& Age            & -0.036 & -0.044 & 0.008 & $[-0.061, 0.077]$ \\
& Educ Num       &  0.162 &  0.149 & 0.013 & $[-0.121, 0.147]$ \\
& Hours/Week     & -0.045 & -0.049 & 0.004 & $[-0.062, 0.070]$ \\
& Marital status &  0.625 &  0.602 & 0.023 & $[-0.110, 0.156]$ \\
\midrule
\multirow{4}{*}{LSP}
& Age            & 0.031 & 0.014 & 0.017 & $[-0.064, 0.098]$ \\
& Educ Num       & 0.132 & 0.131 & 0.001 & $[-0.100, 0.102]$ \\
& Hours/Week     & 0.015 & 0.004 & 0.011 & $[-0.070, 0.092]$ \\
& Marital status & 0.528 & 0.509 & 0.019 & $[-0.094, 0.132]$ \\
\bottomrule
\end{tabular}
\caption{Contributions of features to group-ESL after fairness correction.}
\label{tab:contributions_with_ci}
\end{table}

\section{Conclusion}\label{conclusion}
 
We have proposed  a unified attribution-based framework for assessing and testing group (un)fairness in binary classification models. By formulating standard group fairness criteria, namely independence, separation, and sufficiency in terms of  ESL values computed over classification metrics, we have provided a  bridge between axiomatic attribution theory and formal statistical notions of group fairness. Building on this, we have introduced a two-stage ESL decomposition: first between groups, then in a second stage we have decomposed each group contribution across features. This decomposition enables  a more granular and interpretable diagnosis of the sources of unfairness while remaining consistent with the  axiomatic foundations of explainability, namely efficiency, symmetry, and linearity embodied by the Shapley value in particular and the ESL values.

From a practical perspective, the proposed approach offers computational tractability and robustness, achieved through the Equal Surplus value as a linear time alternative to classical Shapley computations. This makes large scale fairness analysis feasible for complex and realistic problems. The application to the Census Income dataset shows that our methodology  (i) detect  violations of group fairness through first-stage ESL value and allows for formal statistical testing; (ii) employs the two-stage ESL to identify  the specific features that drive discrepancies in group contributions, thereby not only recovering known sources of gender based disparity such as ``Marital status", but also highlighting additional drivers such as  ``Age", and ``hours-per-week". This avoids the conclusion that unfairness  arises only  from purely proxy features for gender and this finally  provides a statistical test that is more time efficient than bootstrap test.

We have extended the two-stage ESL framework beyond the case of a single binary sensitive feature (See Appendix \ref{Appendix-A}). Because ESL values are defined independently of the number of groups, this approach can be extended to account for multiple sensitive features and their intersections, thereby supporting a systematic fairness gerrymandering analysis within the  same axiomatic and computational framework.

\subsection*{Declaration of generative AI use} During the preparation of this work the authors used  GPT6 in order to check some parts of the python code, and English editing. After using this tool, the authors reviewed and edited the content as needed and take full responsibility for the content of the published article.

\printbibliography

\appendix
\section{Appendix A: ESL decompositions}\label{Appendix-A}
In this appendix, we provide the proofs of the theoretical results stated in Sections \ref{FER} and \ref{ULES}.

\subsection{Results on first stage and two-stage decompositions}

\subsubsection*{Proof of Remark \ref{equiv1.3}}
The Shapley value applied to the characteristic function $v^{(m)}(X \big|_{\s_{\mathcal{A}}} (\mathcal{N})) $ such that $\s_{\mathcal{A}}\subseteq \mathcal{A}$ is defined as:
\begin{equation}\notag
\varphi^{Sh}_{g}(\mathbf{X}, v^{(m)}):=\sum_{\s_{\mathcal{A}} \subseteq \mathcal{A} \setminus \{g\}}  P(\s_{\mathcal{A}}) \Big(v^{(m)}(X \big|_{\s_{\mathcal{A}} \cup \{g\}} (\mathcal{N})
) - v^{(m)}(X \big|_{\s_{\mathcal{A}}} (\mathcal{N})) \Big)
\end{equation}
The Shapley value yields, for the  sensitive feature level associated to men:
\begin{align*}
\varphi^{Sh}_{g=1}(\mathbf{X}, v^{(m)})=  \frac{1}{2}\Big(v^{(m)}(X \big|_{\emptyset \cup \{1\}} (\mathcal{N})
) - v^{(m)}(X \big|_{\emptyset} (\mathcal{N})) \Big)+ \frac{1}{2}\Big(v^{(m)}(X \big|_{\{2\} \cup \{1\}} (\mathcal{N})
) - v^{(m)}(X \big|_{\{2\}} (\mathcal{N})) \Big)
\end{align*}
Similarly, the contribution of the women group is:
\begin{align*}
\varphi^{Sh}_{g=2}(\mathbf{X}, v^{(m)})=  \frac{1}{2}\Big(v^{(m)}(X \big|_{\emptyset \cup \{2\}} (\mathcal{N})
) - v^{(m)}(X \big|_{\emptyset} (\mathcal{N})) \Big)+ \frac{1}{2}\Big(v^{(m)}(X \big|_{\{1\} \cup \{2\}} (\mathcal{N})
) - v^{(m)}(X \big|_{\{1\}} (\mathcal{N})) \Big)
\end{align*}
\begin{enumerate}
\item Let us assume that  $f_t$ respects \eqref{IND}. We have by definition:

\begin{align*}
v^{(SR)}(X \big|_{\{1\}} (\mathcal{N})) &= 
\frac{\mathbb{P}(\hat{Y} = 1 \mid A =0)}{\mathbb{P}(\hat{Y} = 1 \mid A \in \emptyset)}(1-\mathds{1}_{|\{1\}|})=\frac{\mathbb{P}(\hat{Y} = 1 \mid A =0 )}{\mathbb{P}(\hat{Y} = 1 \mid A \in \emptyset)}  \\
v^{(SR)}(X \big|_{\{2\}} (\mathcal{N})) &= 
\frac{\mathbb{P}(\hat{Y} = 1 \mid A =1)}{\mathbb{P}(\hat{Y} = 1 \mid A \in \emptyset)}(1-\mathds{1}_{|\{2\}|})=\frac{\mathbb{P}(\hat{Y} = 1 \mid A =1 )}{\mathbb{P}(\hat{Y} = 1 \mid A \in \emptyset)} 
\end{align*}

Therefore, by \eqref{IND} it comes,
\begin{align*}
 v^{(SR)}(X \big|_{\{1\}} (\mathcal{N})) - v^{(SR)}(X \big|_{\{2\}} (\mathcal{N})) =\frac{\mathbb{P}(\hat{Y} = 1 \mid A =0)}{\mathbb{P}(\hat{Y} = 1 \mid A \in \emptyset)}  - \frac{\mathbb{P}(\hat{Y} = 1 \mid A =1)}{\mathbb{P}(\hat{Y} = 1 \mid A \in \emptyset)}  =0
\end{align*}
Then,
\begin{align*}
   \varphi^{Sh}_{g=1}(\mathbf{X}, v^{(SR)})-\varphi^{Sh}_{g=2}(\mathbf{X}, v^{(SR)})&=0
\end{align*}

Conversely, if $\varphi_{g=1}^{Sh}(\mathbf{X}, v^{(SR)})=\varphi_{g=2}^{Sh}(\mathbf{X}, v^{(SR)})$, it is easy to show that \eqref{IND} holds directly since  $v^{(SR)}(X \big|_{\{1\}}(\mathcal{N}))-v^{(SR)}(X \big|_{\{2\}}(\mathcal{N}))=0$.

\item Let us assume that   $f_t$ respects \eqref{SEP}. We have by definition, 
\begin{align*}
v^{(TPR )}(X \big|_{\{1\}} (\mathcal{N})) &= 
\frac{\mathbb{P}(\hat{Y} = 1 \mid Y=1,A =0)}{\mathbb{P}(\hat{Y} = 1 \mid Y=1, A \in \emptyset)}(1-\mathds{1}_{|\{1\}|})=\frac{\mathbb{P}(\hat{Y} = 1 \mid Y=1,A =0 )}{\mathbb{P}(\hat{Y} = 1 \mid Y=1,A \in \emptyset)}  \quad \\
v^{(TPR )}(X \big|_{\{2\}} (\mathcal{N})) &= 
\frac{\mathbb{P}(\hat{Y} = 1 \mid Y=1,A =1)}{\mathbb{P}(\hat{Y} = 1 \mid Y=1, A \in \emptyset)}(1-\mathds{1}_{|\{2\}|})=\frac{\mathbb{P}(\hat{Y} = 1 \mid Y=1,A =1 )}{\mathbb{P}(\hat{Y} = 1 \mid Y=1,A \in \emptyset)}  \quad 
\end{align*}

Then, by \eqref{SEP} it comes,
\begin{align*}
 v^{(TPR )}(X \big|_{\{1\}} (\mathcal{N})) - v^{(TPR )}(X \big|_{\{2\}} (\mathcal{N})) =\frac{\mathbb{P}(\hat{Y} = 1 \mid Y=1,A =0)}{\mathbb{P}(\hat{Y} = 1 \mid Y=1,A \in \emptyset)}  -\frac{\mathbb{P}(\hat{Y} = 1 \mid Y=1,A =1 )}{\mathbb{P}(\hat{Y} = 1 \mid Y=1,A \in \emptyset)} =0
\end{align*}
Then,
\begin{align*}
   \varphi^{Sh}_{g=1}(\mathbf{X}, v^{(TPR)})-\varphi^{Sh}_{g=2}(\mathbf{X}, v^{(TPR)})&=0
\end{align*}
The same logic holds when using $v^{(FPR )}$. Conversely, if $\varphi_{g=1}^{Sh}(\mathbf{X}, v^{(TPR)})=\varphi_{g=2}^{Sh}(\mathbf{X}, v^{(TPR )})$, and $\varphi_{g=1}^{Sh}(\mathbf{X}, v^{(FPR)})=\varphi_{g=2}^{Sh}(\mathbf{X}, v^{(FPR )})$ it is easy to show that \eqref{SEP} holds.

\item  Let assume that  $f_t$ respects (\ref{SUF}).  We have by definition, 
\begin{align*}
v^{(PPV)}(X \big|_{\{1\}} (\mathcal{N})) &= 
\frac{\mathbb{P}(Y = 1 \mid \hat{Y}=1, A =0 )}{\mathbb{P}(Y = 1 \mid \hat{Y}=1,A \in \emptyset)}(1-\mathds{1}_{|\{1\}|})=\frac{\mathbb{P}(Y = 1 \mid \hat{Y}=1,A =0)}{\mathbb{P}(Y = 1 \mid \hat{Y}=1,A \in \emptyset)}  \\
v^{(PPV)}(X \big|_{\{2\}} (\mathcal{N})) &= 
\frac{\mathbb{P}(Y = 1 \mid \hat{Y}=1, A =1 )}{\mathbb{P}(Y = 1 \mid \hat{Y}=1,A \in \emptyset)}(1-\mathds{1}_{|\{2\}|})=\frac{\mathbb{P}(Y = 1 \mid \hat{Y}=1,A =1)}{\mathbb{P}(Y = 1 \mid \hat{Y}=1,A \in \emptyset)} 
\end{align*}

Then, by \eqref{SUF} we have,
\begin{align*}
 v^{(PPV)}(X \big|_{\{1\}} (\mathcal{N})) - v^{(PPV)}(X \big|_{\{2\}} (\mathcal{N})) =\frac{\mathbb{P}(Y = 1 \mid \hat{Y}=1,A =0)}{\mathbb{P}(Y = 1 \mid \hat{Y}=1,A \in \emptyset)}   -\frac{\mathbb{P}(Y = 1 \mid \hat{Y}=1,A =1)}{\mathbb{P}(Y = 1 \mid \hat{Y}=1,A \in \emptyset)}   =0
\end{align*}

Therefore,
\begin{align*}
   \varphi^{Sh}_{g=1}(\mathbf{X}, v^{(PPV)})-\varphi^{Sh}_{g=2}(\mathbf{X}, v^{(PPV )})&=0
\end{align*}

The same logic holds when using $v^{(FOR )}$. Conversely, if $\varphi_{g=1}^{Sh}(\mathbf{X}, v^{(PPV)})=\varphi_{g=2}^{Sh}(\mathbf{X}, v^{(PPV )})$, and $\varphi_{g=1}^{Sh}(\mathbf{X}, v^{(FOR)})=\varphi_{g=2}^{Sh}(\mathbf{X}, v^{(FOR )})$, it is easy to show that  \eqref{SUF} holds.

\end{enumerate}

\subsubsection*{Proof of Proposition \ref{equiv-2}}

\noindent Suppose that the relevant fairness criterion holds on
\(\mathbf X(R)\) for every nonempty \(R\subseteq\mathcal S\).
By Remark~\ref{equiv1.3},
\[
\varphi_g^{\mathrm{Sh}}
\bigl(\mathbf X(R),v^{(m)}\bigr)
=
\frac12\,v^{(m)}\bigl(\mathbf X(R)\bigr),
\
\forall g\in\mathcal A.
\]
Therefore, for every nonempty \(T\subseteq\mathcal S\), linearity of
the Shapley value gives
\[
\mathcal C_{g,T}^k\!\bigl(v^{(m)}\bigr)
=
\frac12\,
\varphi_{k,T}^{\mathrm{Sh}}
\bigl(\mathbf X(T),v^{(m)}\bigr),
\
\forall g\in\mathcal A,\quad \forall k\in T.
\]
Conversely, summing these identities over \(k\in T\) and invoking
efficiency yields
\[
\varphi_g^{\mathrm{Sh}}
\bigl(\mathbf X(T),v^{(m)}\bigr)
=
\frac12\,v^{(m)}\bigl(\mathbf X(T)\bigr),
\
\forall g\in\mathcal A.
\]
Remark~\ref{equiv1.3} then implies the corresponding fairness
criterion on \(\mathbf X(T)\), which completes the proof.

\subsubsection*{Proof of Theorem \ref{LES1.3}}

We work with the case of binary sensitive feature  with $v^{(m)}(X \big|_{\{\emptyset\}}(\mathcal{N}))=0$, then
\begin{align*}
\varphi_{g=1}^{ESL}(\mathbf{X}, v^{(m)}) &= \frac{1}{2}\left( b_{2}v^{(m)}(X \big|_{\mathcal{A}}(\mathcal{N})) - b_{1}v^{(m)}(X \big|_{\{2\}}(\mathcal{N})) \right) +  \frac{1}{2} b_{1}v^{(m)}(X \big|_{\{1\}}(\mathcal{N})) \\
\varphi_{g=2}^{ESL}(\mathbf{X}, v^{(m)}) &= \frac{1}{2}\left( b_{2}v^{(m)}(X \big|_{\mathcal{A}}(\mathcal{N})) - b_{1}v^{(m)}(X \big|_{\{1\}}(\mathcal{N})) \right) +  \frac{1}{2} b_{1}v^{(m)}(X \big|_{\{2\}}(\mathcal{N}))
\end{align*}
Thus, 
\begin{align*}
\varphi_{g=1}^{ESL}(\mathbf{X}, v^{(m)})-\varphi_{g=2}^{ESL}(\mathbf{X}, v^{(m)})= \frac{1}{2}&(b_{1}(v^{(m)}(X \big|_{\{1\}}(\mathcal{N}))-v^{(m)}(X \big|_{\{2\}}(\mathcal{N})))\\
&+ b_{1}(v^{(m)}(X \big|_{\{1\}}(\mathcal{N}))-v^{(m)}(X \big|_{\{2\}}(\mathcal{N}))))
\end{align*}
For each ESL value, since $b_2=1$, setting $b_1 \neq 0$ yields,
\begin{equation} \label{LESdiff}\notag
\varphi_{g=1}^{ESL}(\mathbf{X}, v^{(m)})-\varphi_{g=2}^{ESL}(\mathbf{X}, v^{(m)})=b_1(v^{(m)}(X \big|_{\{1\}}(\mathcal{N}))-v^{(m)}(X \big|_{\{2\}}(\mathcal{N})))
\end{equation}
Assume that either \eqref{IND}, \eqref{SEP}, or \eqref{SUF} holds. Then,
\begin{align*}
v^{(m)}(X \big|_{\{1\}}(\mathcal{N}))= v^{(m)}(X \big|_{\{2\}}(\mathcal{N}))= v^{(m)}(X \big|_{\mathcal{A}}(\mathcal{N})).
\end{align*}
Therefore,
\[
\varphi_{g=1}^{ESL}(X, v^{(m)})
=
\varphi_{g=2}^{ESL}(X, v^{(m)}),
\]
and, in particular,
\[
\varphi_{g=1}^{ESL}(X, v^{(m)})
=
\frac{b_2}{2}\, v^{(m)}(X \big|_{\mathcal{A}}(\mathcal{N})).
\]

Conversely, assume that
\[
\varphi_{g=1}^{ESL}(X, v^{(m)})
=
\varphi_{g=2}^{ESL}(X, v^{(m)}).
\]
Then, we have either
\[
v^{(m)}(X \big|_{\{1\}}(\mathcal{N}))
-
v^{(m)}(X \big|_{\{2\}}(\mathcal{N}))
= 0
\quad \text{or} \quad
b_1 = 0.
\]
Since $b_1 \neq 0$, it follows that
\[
v^{(m)}(X \big|_{\{1\}}(\mathcal{N}))
=
v^{(m)}(X \big|_{\{2\}}(\mathcal{N}).
\]
Therefore, \eqref{IND}, \eqref{SEP}, or \eqref{SUF} holds.

\subsection{Generalization: Proof of the multi-stage ESL decomposition}

\begin{theorem}\label{theromMULTI} \textbf{\emph{(Multi-stage ESL)}}  Let $f_t$ be a binary  classifier and $\mathcal{A}_j$ a binary set of sensitive feature levels, such that $b_1\neq 0$. The multi-stage ESL value $\varphi_{k}^{ESL} \circ ... \circ \varphi_{a^{i_1}_{1},...,a^{i_s}_{s}}^{ESL}(\mathbf{X}, v^{(m)})$ for every nonempty $T\subseteq \s$, every $k\in T$ and every  $(a^{i_1}_{1},\dots,a^{i_s}_{s}) \in \mathcal{A}_1\times\dots\times\mathcal{A}_s $   yields the following equivalences:
\begin{enumerate}
\item[\emph{(i)}]
The classifier \(f_t\) respects \eqref{IND} on \(\mathbf X(T)\) for
every nonempty \(T\subseteq\mathcal S\) if and only if
\[
\mathcal{C}^k_{a^{i_1}_{1},...,a^{i_s}_{s},T}\!\bigl(v^{(SR)}\bigr)
= (\frac{b_2}{2})^s\varphi_{k,T}^{ESL}(\mathbf{X}(T), v^{(SR)}), \ \forall (a^{i_1}_{1},\dots,a^{i_s}_{s})\in\mathcal A_1\times\dots\times\mathcal A_s, \
\forall k\in T\neq \emptyset.
\]
\item[\emph{(ii)}]
The classifier \(f_t\) respects \eqref{SEP} on \(\mathbf X(T)\) for
every nonempty \(T\subseteq\mathcal S\) if and only if
\[
\begin{cases}
\displaystyle
\mathcal{C}^k_{a^{i_1}_{1},...,a^{i_s}_{s},T}\!\bigl(v^{(TPR)}\bigr)
= (\frac{b_2}{2})^s\varphi_{k,T}^{ESL}(\mathbf{X}(T), v^{(TPR)}), \ \forall (a^{i_1}_{1},\dots,a^{i_s}_{s})\in\mathcal A_1\times\dots\times\mathcal A_s, \
\forall k\in T\neq \emptyset. \\
\displaystyle
\mathcal{C}^k_{a^{i_1}_{1},...,a^{i_s}_{s},T}\!\bigl(v^{(FPR)}\bigr)
= (\frac{b_2}{2})^s\varphi_{k,T}^{ESL}(\mathbf{X}(T), v^{(FPR)}), \ \forall (a^{i_1}_{1},\dots,a^{i_s}_{s})\in\mathcal A_1\times\dots\times\mathcal A_s, \
\forall k\in T\neq \emptyset.
\end{cases}
\]
\item[\emph{(iii)}]
The classifier \(f_t\) respects \eqref{SUF} on \(\mathbf X(T)\) for
every nonempty \(T\subseteq\mathcal S\) if and only if
\[
\begin{cases}
\displaystyle
\mathcal{C}^k_{a^{i_1}_{1},...,a^{i_s}_{s},T}\!\bigl(v^{(PPV)}\bigr)
= (\frac{b_2}{2})^s\varphi_{k,T}^{ESL}(\mathbf{X}(T), v^{(PPV)}), \ \forall (a^{i_1}_{1},\dots,a^{i_s}_{s})\in\mathcal A_1\times\dots\times\mathcal A_s, \
\forall k\in T\neq \emptyset. \\
\displaystyle
\mathcal{C}^k_{a^{i_1}_{1},...,a^{i_s}_{s},T}\!\bigl(v^{(FOR)}\bigr)
= (\frac{b_2}{2})^s\varphi_{k,T}^{ESL}(\mathbf{X}(T), v^{(FOR)}), \ \forall (a^{i_1}_{1},\dots,a^{i_s}_{s})\in\mathcal A_1\times\dots\times\mathcal A_s, \
\forall k\in T\neq \emptyset.
\end{cases}
\]
\end{enumerate}
\end{theorem}
\begin{proof}

Let assume that we have two sensitive features with binary levels, gender (men, women) and ethnicity (majority, minority). For the ease of exposition, let $\mathcal{G}= \{1,2\}$ be the set of gender levels where $g\in \mathcal{G}$. Similarly, let $\mathcal{R}= \{1,2\}$  be the set of ethnicity levels where $r\in\mathcal{R}$. Suppose that the fairness criterion holds across the four  intersectional  groups on $\X(Q)$, for every nonempty $Q\subseteq T$. In the first stage,  ESL value is applied  to determine the contribution of gender levels. Following Theorem \ref{LES1.3} we have: 
\begin{align*}
\varphi_{g=1}^{ESL}(\X(T), v^{(m)}) &= \frac{1}{2}\left( b_{2}v^{(m)}(X \big|_{\mathcal{G}}(T)) - b_{1}v^{(m)}(X \big|_{\{2\}}(T)) \right) +  \frac{1}{2} b_{1}v^{(m)}(X \big|_{\{1\}}(T)) \\
\varphi_{g=2}^{ESL}(\X(T), v^{(m)}) &= \frac{1}{2}\left( b_{2}v^{(m)}(X \big|_{\mathcal{G}}(T)) - b_{1}v^{(m)}(X \big|_{\{1\}}(T)) \right) +  \frac{1}{2} b_{1}v^{(m)}(X \big|_{\{2\}}(T))
\end{align*}
Then in a second stage, ESL value is applied on ethnicity levels  to determine their contribution to $\varphi_{g=1}^{ESL}(\X(T), v^{(m)})$. 
For simplicity, let take $\z_g \equiv X \big|_{\mathcal{G}}$, $\z_1\equiv X \big|_{\{1\}}$ and $\z_2\equiv X \big|_{\{2\}}$ as the features defined on each test data defined with respect to  gender levels $g\in \mathcal{G}$. Thus,
\begin{align*}
\varphi_{g=1,r=1}^{ESL}(\z, v^{(m)}) &= \frac{1}{4}b_{2}\left( b_2v^{(m)}(\z_g \big|_{\mathcal{R}}(T)) -b_1v^{(m)}(\z_g  \big|_{\{2\}}(T)) +b_1v^{(m)}(\z_g  \big|_{\{1\}}(T))  \right)\\
&-\frac{1}{4}b_{1}\left( b_2v^{(m)}(\z_2 \big|_{\mathcal{R}}(T)) -b_1v^{(m)}(\z_2  \big|_{\{2\}}(T)) +b_1v^{(m)}(\z_2  \big|_{\{1\}}(T))  \right)\\
&+\frac{1}{4}b_{1}\left( b_2v^{(m)}(\z_1 \big|_{\mathcal{R}}(T)) -b_1v^{(m)}(\z_1  \big|_{\{2\}}(T)) +b_1v^{(m)}(\z_1  \big|_{\{1\}}(T))  \right)
\end{align*}
Similarly, we find:
\begin{align*}
\varphi_{g=1,r=2}^{ESL}(\z, v^{(m)})
&= \frac{1}{4}b_{2}\Bigl(
    b_2v^{(m)}(\z_g \big|_{\mathcal{R}}(T))
    - b_1v^{(m)}(\z_g \big|_{\{1\}}(T))
    + b_1v^{(m)}(\z_g \big|_{\{2\}}(T))
\Bigr) \\
&\quad - \frac{1}{4}b_{1}\Bigl(
    b_2v^{(m)}(\z_2 \big|_{\mathcal{R}}(T))
    - b_1v^{(m)}(\z_2 \big|_{\{1\}}(T))
    + b_1v^{(m)}(\z_2 \big|_{\{2\}}(T))
\Bigr) \\
&\quad + \frac{1}{4}b_{1}\Bigl(
    b_2v^{(m)}(\z_1 \big|_{\mathcal{R}}(T))
    - b_1v^{(m)}(\z_1 \big|_{\{1\}}T))
    + b_1v^{(m)}(\z_1 \big|_{\{2\}}(T))
\Bigr)
\end{align*}

This yields the following difference:
\begin{align*}
\varphi_{g=1,r=1}^{ESL}(\z, v^{(m)})
&-\varphi_{g=1,r=2}^{ESL}(\z, v^{(m)}) \\
&= \frac{b_1b_2}{2}
\left(
    v^{(m)}\left(\z_g\big|_{\{1\}}(T)\right)
    - v^{(m)}\left(\z_g\big|_{\{2\}}(T)\right)
\right) \\
&\quad + \frac{b_1^2}{2}\Bigl[
    \left(
        v^{(m)}\left(\z_1\big|_{\{1\}}(T)\right)
        - v^{(m)}\left(\z_1\big|_{\{2\}}(T)\right)
    \right) \\
&\qquad\qquad
    - \left(
        v^{(m)}\left(\z_2\big|_{\{1\}}(T)\right)
        - v^{(m)}\left(\z_2\big|_{\{2\}}(T)\right)
    \right)
\Bigr].
\end{align*}

Let assume that either \eqref{IND}, \eqref{SEP}, or \eqref{SUF} holds at the ethnicity level on $\X(T)$, within each $\z_g,\z_1$ and $\z_2$. Hence, for $(.)\in \{g,1,2\}$,
\begin{align*}
v^{(m)}(\z_{(.)} \big|_{\{1\}}(T)= v^{(m)}(\z_{(.)} \big|_{\{2\}}(T))
&= v^{(m)}(\z_{(.)} \big|_{\mathcal{R}}(T)).
\end{align*}
Therefore,
\[
\varphi_{g=1,r=1}^{ESL}(\z, v^{(m)})
=
\varphi_{g=1,r=2}^{ESL}(\z, v^{(m)}),
\]
such that
$$
\varphi_{g=1,r=1}^{ESL}(\z, v^{(m)})=\frac{1}{4}\Big(b_2^2v^{(m)}(\z_g\big|_{\mathcal{R}}(T)-b_1b_2v^{(m)}(\z_2\big|_{\mathcal{R}}(T)+b_1b_2v^{(m)}(\z_1\big|_{\mathcal{R}}(T)\Big)
$$  
If we assume that  group fairness under (\ref{IND}), (\ref{SEP}) or (\ref{SUF}) holds at gender level, the contribution is reduced to:
$$
\varphi_{g=1,r=1}^{ESL}(\z, v^{(m)})=\frac{1}{4}b_2^2v^{(m)}(\z_g\big|_{\mathcal{R}}(T))
$$ 
The same procedure is applied to determine $\varphi_{g=2,r=1}^{ESL}(\z, v^{(m)})$ and $\varphi_{g=2,r=2}^{ESL}(\z, v^{(m)})$.
Note that, under group fairness, $\varphi_{g=1}^{ESL}(\X(T), v^{(m)})=\varphi_{g=2}^{ESL}(\X(T), v^{(m)})$. By the axiom of efficiency, $$\varphi_{g=1,r=1}^{ESL}(\z, v^{(m)})+\varphi_{g=1,r=2}^{ESL}(\z, v^{(m)})=\varphi_{g=2,r=1}^{ESL}(\z, v^{(m)})+\varphi_{g=2,r=2}^{ESL}(\z, v^{(m)})$$ 
Since $\varphi_{g=1,r=1}^{ESL}(\z, v^{(m)})=\varphi_{g=1,r=2}^{ESL}(\z, v^{(m)})$, and $\varphi_{g=2,r=1}^{ESL}(\z, v^{(m)})$=$\varphi_{g=2,r=2}^{ESL}(\z, v^{(m)})$,  then
$$\varphi_{g=1,r=1}^{ESL}(\z, v^{(m)})=\varphi_{g=2,r=1}^{ESL}(\z, v^{(m)})$$ and, $$\varphi_{g=1,r=2}^{ESL}(\z, v^{(m)})=\varphi_{g=2,r=2}^{ESL}(\z, v^{(m)})$$ 
Then, group fairness is ensured for all nested groups. 

Similarly to proof of Remark \ref{equiv1.3}, since the contribution of the ethnicity level $r\in \mathcal{R}$ is $\varphi_{g=1,r=1}^{ESL}(\z, v^{(m)})=\varphi_{g=1,r=2}^{ESL}(\z, v^{(m)})=\frac{1}{4}b_2^2v^{(m)}(\z_g\big|_{\mathcal{R}}(T))= \frac{1}{4}b_2^2v^{(m)}(X \big|_{{
\mathcal{G}}\cup{\mathcal{R}}}(T)))$. Furthermore, since the group fairness criterion is assumed on $\X(Q)$, for every nonempty $Q\subseteq T$, the same computation with T is replaced by Q gives $\varphi_{g=1,r=1}^{ESL}(\z, v^{(m)})=\varphi_{g=1,r=2}^{ESL}(\z, v^{(m)})=\frac{1}{4}b_2^2v^{(m)}(\z_g\big|_{\mathcal{R}}(Q))= \frac{1}{4}b_2^2v^{(m)}(X \big|_{{
\mathcal{G}}\cup{\mathcal{R}}}(Q)))$, where $\varphi_{g,r}^{ESL}(\z, v^{(m)})$ is now computed on the features in $Q$. Then, applying the ESL value to the characteristic function $\varphi_{g=1,r}^{ESL}(\z, v^{(m)})$, which is independent of the ethnicity and gender levels, yields by the additivity of the ESL, $\mathcal{C}^k_{g=1,r,T}(v^{(m)})=\frac{1}{4}b_2^2\varphi_{k,T}^{ESL}(\mathbf{X}(T), v^{(m)})$. Similarly we find  $\mathcal{C}^k_{g=2,r,T}(v^{(m)})=\frac{1}{4}b_2^2\varphi_{k,T}^{ESL}(\mathbf{X}(T), v^{(m)})$.

When having multiple sensitive features with binary levels, by recurrence, the contribution of the $s^{th}$ sensitive feature with level $j$ is expressed as, 
$$\varphi_{a^{i_1}_{1},...,a^{i_s}_{s}}^{ESL}(\mathbf{X}(T), v^{(m)})=\Big(\frac{b_2}{2}\Big)^sv^{(m)}(X \big|_{{\mathcal{A}_1}\cup...\cup{\mathcal{A}_s}}(T)))
$$ 
Therefore, $\mathcal{C}^k_{a^{i_1}_{1},...,a^{i_s}_{s},T}(v^{(m)})=(\frac{b_2}{2})^s\varphi_{k,T}^{ESL}(\mathbf{X}(T), v^{(m)})$. The other direction $(\Leftarrow)$ holds directly, and this ends the proof. 
\end{proof}

\section{Appendix B: Testing equal opportunity}\label{Appendix-B}

We provide a statistical fairness test with respect to equal opportunity  using the ESL values  proposed in  Theorem \ref{LES1.3} and Corollary \ref{equiv-41}. Recall that equal opportunity \eqref{EOD} is a relaxation of \eqref{SEP}:
\begin{equation}\tag{EOD}\label{EOD}
 \mathbb P(\hat{Y}=1|Y=1,A=1)= \mathbb P(\hat{Y}=1|Y=1,A=2)
\end{equation}
Importantly, the framework can be adapted to test other group fairness criteria, namely \eqref{IND}, \eqref{SEP} and \eqref{SUF}  by modifying the characteristic function $v^{(m)}$.

\subsection{First stage ESL statistical test} 

To assess whether group fairness holds with respect to equal opportunity \eqref{EOD}, the following null
hypothesis derived from Theorem \ref{LES1.3} is introduced:
\begin{equation}\notag
\left\|
\begin{array}{l}
H_0 : \varphi_{g=1}^{ESL}(\mathbf{X}, v^{(TPR)}) = \varphi_{g=2}^{ESL}(\mathbf{X}, v^{(TPR)}) \ \Longleftrightarrow \  \eqref{EOD} 
\\ 
H_1 : \varphi_{g=1}^{ESL}(\mathbf{X}, v^{(TPR)}) \neq \varphi_{g=2}^{ESL}(\mathbf{X}, v^{(TPR)}) \ \Longleftrightarrow \  \lnot\eqref{EOD} 
\end{array} 
\right.
\end{equation} 
In other words, testing $\varphi_{g=1}^{ESL}(\mathbf{X},v^{(TPR)}) = \varphi_{g=2}^{ESL}(\mathbf{X}, v^{(TPR)})$ assesses whether \eqref{EOD} holds. The rejection of $H_0$ implies the violation of $\eqref{EOD}$. For simplicity, let $\widehat{\varphi}_{g}^{ESL}:=\widehat{\varphi}_{g}^{ESL}(\mathbf{X}, v^{(TPR)})$. We establish the asymptotic distribution of the difference $\widehat{\varphi}_{1}^{ESL}-\widehat{\varphi}_{2}^{ESL}$.

\begin{proposition}\label{stat1}
Fix the common normalization baseline at $\beta_m=1/2$ for $m=TPR$,
and let $b_1\neq 0$ be determined by Proposition \ref{LEStheo}. Condition on the training sample
and assume that the evaluation observations are i.i.d. and
independent of it.
Let $n_1^+$ and $n_2^+$ denote the numbers of actual positives
among men and women, respectively, and assume that
$n_g^+/n\to\lambda_g>0$ as the total evaluation sample size
$n\to\infty$. Under the null hypothesis that \eqref{EOD} holds,
assume that the common true positive rate $p$ satisfies $0<p<1$.
Then
\[
Z=
\frac{\widehat{\varphi}_{1}^{ESL}
      -\widehat{\varphi}_{2}^{ESL}}
{\sqrt{
4b_1^2\hat{p}_{\mathcal A}(1-\hat{p}_{\mathcal A})
\left(\frac{1}{n_1^+}+\frac{1}{n_2^+}\right)
}}
\xrightarrow{d}\mathcal N(0,1),
\]
where $\hat{p}_{\mathcal A}$ is the pooled empirical true positive
rate for both groups.
\end{proposition}

\begin{proof}
In this proof, we take $m=TPR$ and fix the normalization baseline
at $\beta_m=1/2$. More generally, the argument applies to any
$m\in\{SR,TPR,FPR,PPV,FOR\}$ with a fixed baseline $\beta_m>0$
common to both groups and the corresponding denominator counts
$n_g^+$. Retaining $\beta_m$ explicitly replaces the factors $2$
in the group contributions and $4$ in the variance by $1/\beta_m$
vand $1/\beta_m^2$, respectively.

Let \( X \big|_{g}(\mathcal{N}) = (X_1, \dots, X_N)\big|_{g}\) and \( X \big|_{\mathcal{A}}(\mathcal{N}) = (X_1, \dots, X_N)\big|_{\mathcal{A}}\) denote the  random feature vector associated  to level $g\in\{1,2\}$, and combined level $\mathcal{A}=\{1,2\}$, respectively. These features are used to evaluate a binary classifier $f_t$ on a specific dataset (\textit{e.g.}, combined data, level data). Let us denote the indicator of true positive when both the ground truth label $Y_{i}$ and the  predicted label $\hat{Y}_{i}$  are  positive as:
\begin{align*}
 I_{i,\mathcal{N}}= \mathds{1}_{\{\hat{Y}_{i}=1,Y_{i}=1\}} 
\end{align*}
Similarly, the indicator of an actual positive is defined as:
\begin{align*}
 J_{i}= \mathds{1}_{\{Y_{i}=1\}} 
\end{align*}
Let   $n_g$ and $n_{\mathcal{A}}$  denote the total number of observations in test data belonging to group $g$ and all groups, respectively, where $n:=n_{\mathcal{A}}\rightarrow\infty$, which implies $n_{g}\rightarrow\infty$ almost surely, since $\mathbb{P}(A=g)>0$.

By the sampling framework, conditionally on the training sample,  $(J_{i},I_{i, \mathcal{N}})$ 
are i.i.d Bernoulli random vectors across the test observations ( and i.i.d. within each group $g$). For the ease of exposition, we define $\alpha_{g,\mathcal{N}}:=\p(\hat{Y}=1,Y=1|A=g)$, $\alpha_{\mathcal{N}}:=\p(\hat{Y}=1,Y=1)$, $\gamma_g:=\p(Y=1|A=g)$ and $\gamma:=\p(Y=1)$. By  the Law of Large Numbers (LLN), we have the following

\begin{align*}
\frac{\sum_{i=1}^{n_g}I_{i, \mathcal{N}}}{n_g}\; &\overset{p}{\longrightarrow} \; \mathbb{E}(I_{i, \mathcal{N}}|A_i=g)=\alpha_{g,\mathcal{N}} \quad \text{and}\quad \frac{\sum_{i=1}^{n_{\mathcal{A}}}I_{i, \mathcal{N}}}{n_{\mathcal{A}}}\; \overset{p}{\longrightarrow} \; \mathbb{E}(I_{i, \mathcal{N}})=\alpha_{\mathcal{N}} \\
\frac{\sum_{i=1}^{n_g}J_{i}}{n_g}\; &\overset{p}{\longrightarrow} \;\mathbb{E}(J_{i}|A_i=g)=\gamma_{g} \quad \text{and}\quad \frac{\sum_{i=1}^{n_{\mathcal{A}}}J_{i}}{n_{\mathcal{A}}}\; \overset{p}{\longrightarrow} \; \mathbb{E}(J_{i})=\gamma 
\end{align*}
Denote $p_g:=\p(\hat{Y}=1|Y=1,A=g)$ and let $\hat{p}_g$ be the plug in estimator of the  true positive rate of group $g$. In particular,
\begin{align*}
\hat{p}_g:=\widehat{TPR}(X \big|_{\{g\}}(\mathcal{N})) =\frac{\frac{\sum_{i=1}^{n_g}I_{i,\mathcal{N}}}{n_g}}{{\frac{\sum_{i=1}^{n_g}J_{i}}{n_g}}} \equiv f(U_{n_g}),
\end{align*}
where $U_{n_g}:=\Big( \frac{\sum_{i=1}^{n_g}J_{i}}{n_g},\frac{\sum_{i=1}^{n_g}I_{i, \mathcal{N}}}{n_g} \Big)    $ is the two-dimensional sample mean with population mean $\theta_g:=(\gamma_g,\alpha_{g,\mathcal{N}})$. By the Continuous Mapping Theorem and assuming that $\gamma_g>0$, we have: 
\begin{align*}
\widehat{TPR}(X \big|_{\{g\}}(\mathcal{N}))=f(U_{n_g}) &\overset{p}{\longrightarrow} \;\;\;f(\theta_g)=\frac{\alpha_{g,\mathcal{N}}}{\gamma_g}=p_g  
\end{align*}
Hence, $\widehat{TPR}(X \big|_{\{g\}}(\mathcal{N})) $ is a consistent estimator of $\p(\hat{Y}=1|Y=1,A=g)$. Moreover, by the Multivariate Central Limit Theorem, 
\begin{align*}
  \sqrt{n_g}\Big(U_{n_g}-\theta_g\Big)\overset{d}{\longrightarrow} \; \mathcal{N}(0,\Sigma_g)   
\end{align*}
where,
\begin{align*}
    \Sigma_g= \begin{pmatrix}
\gamma_g (1 - \gamma_g) & \alpha_{g,\mathcal{N}}- \gamma_g \alpha_{g,\mathcal{N}} \\
\alpha_{g,\mathcal{N}} - \gamma_g \alpha_{g,\mathcal{N}} & \alpha_{g,\mathcal{N}} (1 - \alpha_{g,\mathcal{N}})
\end{pmatrix}
\end{align*}
Therefore,  since $f$ is differentiable at $\theta_g$, then by the Delta method we have,
\[
\sqrt{n_g} \left[ f(U_{n_g}) - f(\theta_g) \right]  \overset{d}{\longrightarrow}\ \mathcal{N} \left( 0, \nabla f(\theta_g)^\top \Sigma_g \nabla f(\theta_g) \right)
\]
where $\nabla f(\gamma_g,\alpha_{g,\mathcal{N}} )
=
\begin{pmatrix}
  -\tfrac{\alpha_{g,\mathcal{N}} }{\gamma_g^2} \\[6pt]
  \tfrac{1}{\gamma_g}
\end{pmatrix}$ and $\nabla f(\theta_g)^\top \Sigma_g \nabla f(\theta_g)= \frac{p_g\,\bigl(1 - p_g\bigr)}{\gamma_g}.$\\
Hence, 
\begin{align*}
   \sqrt{n_g}\Big(\widehat{TPR}(X\big|_{\{g\}}(\mathcal{N}))-p_g\Big)\overset{d}{\longrightarrow} \; \mathcal{N}\Big(0,\frac{p_g(1-p_g)}{\gamma_g}\Big)   
\end{align*}
That is, for sufficiently large $n_g$,
\begin{align*}
 \widehat{TPR}(X\big|_{\{g\}}(\mathcal{N}))\sim \mathcal{N}\Big(p_g,\frac{p_g(1-p_g)}{n^+_g}\Big)   
\end{align*}
where  $n_g^+=\sum_{i=1}^{n_g}J_i$ is the observed number of actual positives in group $g$, since $\frac{n_g^+}{n_g\gamma_g}\overset{p}{\longrightarrow}1$, by Slutsky's lemma we can replace $n_g\gamma_g$ by $n_g^+$.

Following the same steps, we can find that $\widehat{TPR}(X\big|_{\mathcal{A}}(\mathcal{N}))$ is a consistent estimator of $p_{\mathcal{A}}:=\p(\hat{Y}=1|Y=1)$ and by assuming that $\gamma>0$ so that $\frac{n^+_{\mathcal{A}}}{n_{\mathcal{A}}\gamma_{\mathcal{A}}}\overset{p}{\longrightarrow}1$,
\begin{align*}
 \widehat{TPR}(X\big|_{ \mathcal{A}}(\mathcal{N}))\sim \mathcal{N}\Big(p_{\mathcal{A}},\frac{p_{\mathcal{A}}(1-p_{\mathcal{A}})}{n^+_{\mathcal{A}}}\Big)   
\end{align*}
Recall that  $\varphi_{g}^{ESL}( \mathbf{X}, TPR)$ assigns to each data belonging to level $g\in\{1,2\}$ the following group contribution\footnote{The baseline of a random classifier $\frac{1}{2}$ was used in the definition of $v^{m}$. Hence, the use of 2 in the equations for $\varphi_{g=1}^{ESL}(\mathbf{X}, TPR)$ and  $\varphi_{g=2}^{ESL}(\mathbf{X}, TPR)$.}:
\begin{align*}
\varphi_{g=1}^{ESL}\Big(\mathbf{X}, v^{(TPR)}\Big) &= \frac{1}{2}\Big[\left( 2b_{2}TPR(X\big|_{\mathcal{A}}(\mathcal{N})) - 2b_{1}TPR(X\big|_{\{2\}}(\mathcal{N})) \right) +   2b_{1}TPR(X\big|_{\{1\}}(\mathcal{N}))\Big] \\
\varphi_{g=2}^{ESL}\Big(\mathbf{X}, v^{(TPR)}\Big) &=\frac{1}{2}\Big[ \left( 2b_{2}TPR(X\big|_{\mathcal{A}}(\mathcal{N})) - 2b_{1}TPR(X\big|_{\{1\}}(\mathcal{N})) \right) +  2b_{1}TPR(X\big|_{\{2\}}(\mathcal{N}))\Big]
\end{align*}
Since $\varphi_{g}^{ESL}$ is a continuous function of $TPR_g$ and  $TPR_{\mathcal{A}}$, the  plug-in estimator of  $\varphi_{g}^{ESL}(\mathbf{X},v^{(TPR)})$ is  $\widehat{\varphi_{g}}^{ESL}:=\varphi_{g}^{ESL}(\mathbf{X}, v^{(\widehat{TPR})})$. By the Continuous Mapping Theorem, we have
\begin{align*}
\widehat{\varphi_{1}}^{ESL} \;\;&\overset{p}{\longrightarrow}    \;\;b_2p_{\mathcal{A}}+b_1(p_1-p_2)\\
\widehat{\varphi_{2}}^{ESL}\;\; &\overset{p}{\longrightarrow}    \;\;b_2p_{\mathcal{A}}-b_1(p_1-p_2)
\end{align*}
For ease of exposition, we define $\widehat{TPR}_g(\mathcal{N}):=\widehat{TPR}(X\big|_{\{g\}}(\mathcal{N}))$. Since $\widehat{TPR}_g(\mathcal{N})$ is computed on group observations and the groups are disjoint in the test sample. The sampling framework implies that, conditionally on the training sample, and on the labels,  $\widehat{TPR}_1(\mathcal{N})$ and  $\widehat{TPR}_2(\mathcal{N})$ are independent. 

Moreover, $\widehat{TPR}_{\mathcal{A}}(\mathcal{N})=\frac{n_1^+\widehat{TPR}_1(\mathcal{N})+n_2^+\widehat{TPR}_2(\mathcal{N})}{n_{\mathcal{A}}^+}$. Then, by the multivariate Central Limit Theorem and by the Delta method, $(\widehat{TPR}_{\mathcal{A}}(\mathcal{N}),\widehat{TPR}_1(\mathcal{N}),\widehat{TPR}_2(\mathcal{N}))$ is jointly asymptotically normal, and  
$\widehat{\varphi_{g}}^{ESL} $ (a linear combination of those latter) is also  asymptotically normally distributed: 
\begin{align*}
   \widehat{\varphi}_{g=1}^{ESL} \sim \mathcal{N}\Big(b_2p_{\mathcal{A}}+b_1(p_1-p_2),\Big[b_2^2\frac{p_{\mathcal{A}}(1-p_{\mathcal{A}})}{n^+_{\mathcal{A}}}+b_1^2\Big(\frac{p_1(1-p_1)}{n^+_1}+\frac{p_2(1-p_2)}{n^+_2}\Big)  \Big]+2b_1b_2(c_1-c_2)\Big),
\end{align*}
where  $c_j :=\text{Cov}\big(\widehat{TPR}_{\mathcal{A}}(\mathcal{N}),\widehat{TPR}_{j}(\mathcal{N})\big)$ and 
$\text{Cov}\big(\widehat{TPR}_{1}(\mathcal{N}),\widehat{TPR}_{2}(\mathcal{N})\big)=0$. Similarly,  $\widehat{\varphi}_{g=2}^{ESL}$ is normally distributed as follows:
\begin{align*}
  \widehat{\varphi}_{g=2}^{ESL} \sim \mathcal{N}\Big(b_2p_{\mathcal{A}}-b_1(p_1-p_2),\Big[b_2^2\frac{p_{\mathcal{A}}(1-p_{\mathcal{A}})}{n^+_{\mathcal{A}}}+b_1^2\Big(\frac{p_1(1-p_1)}{n^+_1}+\frac{p_2(1-p_2)}{n^+_2}\Big)  \Big]-2b_1b_2(c_1-c_2)\Big) 
\end{align*}
Further, let us define $D:=\widehat{\varphi}_{g=1}^{ESL}-\widehat{\varphi}_{g=2}^{ESL}$. Therefore, we have the following:
\begin{align*}
    D\sim \mathcal{N}\Big(2b_1(p_1-p_2),4b_1^2\Big(\frac{p_1(1-p_1)}{n^+_1}+\frac{p_2(1-p_2)}{n^+_2}\Big)\Big)
\end{align*}

Note that  when \eqref{EOD} is satisfied,   we have $p_1=p_2=p_{\mathcal{A}}$. Since $\hat{p}_{\mathcal{A}}:=\widehat{TPR}_{\mathcal{A}}(\mathcal{N})\overset{p}{\longrightarrow}p_{\mathcal{A}}$, Slutsky's lemma implies that, under  the null hypothesis,  the asymptotic distribution of the standardized test statistic is: 
\begin{align*}
Z:=\frac{D}{\sqrt{4b_1^2 \hat{p}_{\mathcal{A}}(1-\hat{p}_{\mathcal{A}})\Big(\frac{1}{n^+_1}+\frac{1}{n^+_2}\Big)}}\overset{d}{\longrightarrow} \mathcal{N}(0,1)
\end{align*}
Consequently, the null hypothesis is rejected at level $\alpha$ when $|Z|>z_{\alpha/2}$, where $z_{\frac{\alpha}{2}}$ is the $(1-\frac{\alpha}{2})$ quantile of the standard normal distribution $\mathcal{N}(0,1)$.
\end{proof}

\subsection{Second stage ESL statistical test} 

Contrary to the first stage, the second stage focuses on testing the equality in feature contributions across levels. By Corollary \ref{equiv-41},  equal opportunity \eqref{EOD} can be tested for any given $k\in \mathcal{N}$:
\begin{equation}\notag
\left\|
\begin{array}{l}
H_0 : \mathcal{C}^k_{1}(v^{(TPR)}) = \mathcal{C}^k_{2}(v^{(TPR)}) 
\\ 
H_1 : \mathcal{C}^k_{1}(v^{(TPR)}) \neq \mathcal{C}^k_{2}(v^{(TPR)}) 
\end{array}
\right.
\end{equation}

\begin{proposition}\label{stat2}
For $m=TPR$, fix the common normalization baseline at $\beta_m=1/2$,
and let $b_1\neq 0$ be the ESL coefficient determined by
Proposition \ref{LEStheo}. Condition on the training sample
and assume that the evaluation observations are i.i.d. and
independent of it. For a feature $k\in\mathcal S$, consider the null hypothesis
\[
H_0:\quad \mathcal{C}_1^k(v^{(m)})=\mathcal{C}_2^k(v^{(m)}).
\]
Let
\[
\widehat V_k=
\widehat{\operatorname{Var}}\!\left(\widehat{\mathcal{C}}_1^k(v^{(m)})\right)
+
\widehat{\operatorname{Var}}\!\left(\mathcal{C}_2^k(v^{(m)})\right)
-
2\widehat{\operatorname{Cov}}\!\left(
\mathcal{C}_1^k(v^{(m)}),\mathcal{C}_2^k(v^{(m)})
\right).
\]
Assume that the asymptotic variance $\tau_k^2$ of
$\sqrt n(\mathcal{C}_1^k-\mathcal{C}_2^k)$ is strictly positive
and that $n\widehat V_k\xrightarrow{p}\tau_k^2$, where $n$
is the total evaluation sample size. Then, under $H_0$,
\[
Z_k=
\frac{\mathcal{C}_1^k(v^{(m)})-\mathcal{C}_2^k(v^{(m)})}
{\sqrt{\widehat V_k}}
\xrightarrow{d}\mathcal N(0,1).
\]
\end{proposition}

\begin{proof}
In what follows, the classifier is trained on the features in $\mathcal{S}\subseteq \mathcal{N}$, and is  therefore fixed. Let $\hat{Y}_{i,\s}$ be the predicted label using features \( X\big|_{g}(\mathcal{S}) = (X_s, \dots, X_S)\big|_{g}\). Following  the first stage statistical test, the estimated true positive rate on test data belonging to $g$ based on feature vector  $X\big|_{\{g\}}(\mathcal{S})$  is given by:
\begin{align*}
\widehat{TPR}(X\big|_{\{g\}}(\s)) =\frac{\frac{\sum_{i=1}^{n_g}I_{i,\s}}{n_g}}{\frac{\sum_{i=1}^{n_g}J_{i}}{n_g}}
\end{align*}
where $
I_{i,\s}= \mathds{1}_{\{\hat{Y}_{i,\s}=1,Y_{i}=1\}}$ and $ J_{i}= \mathds{1}_{\{Y_{i}=1\}}$. Again,   substituting the feature vector \( X\big|_{g}(\mathcal{N}) = (x_1, \dots, x_N)\big|_{g}\) by \( X\big|_{g}(\mathcal{S}) = (X_s, \dots, X_S)\big|_{g}\), by the Central Limit Theorem, the asymptotic distribution of $ \widehat{TPR}_g(\s)$ for $n_g$ sufficiently large is:                
\begin{align*} 
 \widehat{TPR}_g(\s)\sim \mathcal{N}\Big(p_g(\s),\frac{p_g(\s)(1-p_g(\s))}{n^+_g}\Big)
\end{align*}
where  $p_g(\s):= \p(\hat{Y}_\s=1|Y=1,A=g)\in(0,1)$, $\gamma_g>0$  so that $\frac{n^+_{\mathcal{A}}}{n_g\gamma_g}\overset{p}{\longrightarrow}1$. Similarly, when $n_{\mathcal{A}}$ is large, 
\begin{align*} 
\widehat{TPR}_{\mathcal{A}}(\s) \sim \mathcal{N} \left( p_{\mathcal{A}}(\s), \frac{p_{\mathcal{A}}(\s) \left( 1 - p_{\mathcal{A}}(\s) \right)}{n^+_{\mathcal{A}}} \right)
\end{align*}
where  $p_{\mathcal{A}}(\s):= \p(\hat{Y}_\s=1|Y=1)\in(0,1)$, $\gamma_{\mathcal{A}}>0$ and $\frac{n^+_{\mathcal{A}}}{n_{\mathcal{A}}\gamma_{\mathcal{A}}}\overset{p}{\longrightarrow}1$.

From the sampling framework, conditionally on the training sample, the vectors $(J_i,(I_{i,\mathcal{S}})_{\mathcal{S}\subseteq N})$ are i.i.d. within each group and independent across groups. Then, by the multivariate Central Limit Theorem and by the Delta method, the vectors of all $\widehat{TPR}_{\mathcal{A}}(\mathcal{S}),\widehat{TPR}_1(\mathcal{S})$ and $\widehat{TPR}_2(\mathcal{S})$, $\mathcal{S}\subseteq \mathcal{N}$ is jointly asymptotically normal. 
In the second stage, the aim is to allocate $\hat{\varphi}_{g}^{ESL}(\X, v^{(TPR)})$ across all features. For simplicity, we define $\hat{w}_g(\s):=\hat{\varphi}_{g}^{ESL}(\X(\s), v^{(TPR)})$.
For every coalition $\s$, since the mapping from $\widehat{TPR}_g(\s)$  to $\hat{w}_g(\s )$ is continuous, then 
\begin{align*}
\hat{w}_1(\s ) \;\;&\overset{p}{\longrightarrow}    \;\;b_2p_{\mathcal{A}}(\s )+b_1(p_1(\s)-p_2(\s ))\\
\hat{w}_2(\s )\;\; &\overset{p}{\longrightarrow}    \;\;b_2p_{\mathcal{A}}(\s )-b_1(p_1(\s )-p_2(\s))
\end{align*}
For large $n_g$ and $n_{\mathcal{A}}$, since $\hat{w}_g(\s)$ is a linear combination of normally distributed random variables, $\hat{w}_g(\s)$ follows a normal distribution:
\begin{align*}
   \hat{w}_g(\s) \sim \mathcal{N}\Big(\mu_g(\s),\sigma_g^2(\s)\Big) 
\end{align*}
where, for $g\in \{1,2\}$,
\begin{align*}
\mu_1(\s)&=b_2p_{\mathcal{A}}(\s)+b_1(p_1(\s)-p_2(\s))\\
\mu_2(\s)&=b_2p_{\mathcal{A}}(\s)-b_1(p_1(\s)-p_2(\s))\\
\sigma_1^2(\s)&=b_2^2\frac{p_{\mathcal{A}}(\s)(1-p_{\mathcal{A}}(\s))}{n^+_{\mathcal{A}}}+b_1^2(\frac{p_1(\s)(1-p_1(\s))}{n^+_1}+\frac{p_2(\s)(1-p_2(\s))}{n^+_2})+2b_1b_2(c_1(\s)-c_2(\s))\Big)\\
\sigma_2^2(\s)&=b_2^2\frac{p_{\mathcal{A}}(\s)(1-p_{\mathcal{A}}(\s))}{n^+_{\mathcal{A}}}+b_1^2(\frac{p_1(\s)(1-p_1(\s))}{n^+_1}+\frac{p_2(\s)(1-p_2(\s))}{n^+_2})-2b_1b_2(c_1(\s)-c_2(\s))\Big)
\end{align*}
with $c_g=\text{Cov}\big(\widehat{TPR}_{\mathcal{A}}(\s),\widehat{TPR}_{g}(\s)\big)$.

The estimator of $\mathcal{C}^k_g(v^{(TPR)})$ is given by:
\begin{align*}
\hat{\mathcal{C}}^k_g(v^{(TPR)}) = \sum_{\s \subseteq \mathcal{N} \setminus \{k\}} P(\s) \Big( \hat{W}_g(\s)  \Big) := \sum_{\s \subseteq \mathcal{N} \setminus \{k\}} P(\s) \Big(b_{s+1}\hat{w}_g(\s\cup\{k\}) - b_s \hat{w}_g(\s) \Big)
\end{align*}
Since the  mapping from $\hat{w}_g(\s)$ to $\hat{\mathcal{C}}^k_g$ is also continuous,  applying the Continuous Mapping Theorem a second time leads to, 
\begin{align*}
\hat{\mathcal{C}}^k_g(v^{(TPR)}) \;\;&\overset{p}{\longrightarrow}  \;\;\mathcal{C}^k_g(v^{(TPR)})=\sum_{\s \subseteq \mathcal{N} \setminus \{k\}} P(\s) \Big(b_{s+1}\mu_g(\s\cup\{k\}) - b_s \mu_g(\s) \Big)
\end{align*}
$\hat{w}_g(\s)$ is asymptotically normally
distributed, by linearity,  $\hat{\mathcal{C}}^k_g(v^{(TPR)})$ is also asymptotically normal. Thus, 
\begin{align*}
    \sqrt{n}(\widehat{\mathcal{C}}^k_g(v^{(TPR)})-\widehat{\mathcal{C}}^k_g(v^{(TPR)})) \;\;&\overset{d}{\longrightarrow}\mathcal{N}\Big(0,\sigma^2_{g,k} \Big)  
\end{align*}
Next, we establish the asymptotic variance:
\begin{align*}
 Var\big(\hat{\mathcal{C}}^k_g(v^{(TPR)})\big)&= Var \Big(\sum_{\s \subseteq \mathcal{N} \setminus \{k\}} P(\s) \Big(\hat{W}_g(\s) \Big)\Big) \\
&=\sum_{\s\subseteq\mathcal{N}\setminus \{k\}}\sum_{T\subseteq\mathcal{N}\setminus \{k\}}P(\mathcal{S})P(T) \text{Cov}\left(\hat{W}_g(\s)),\hat{W}_g(T)\right)\\
&=\sum_{S\subseteq\mathcal{N}\setminus \{k\}}P(\mathcal{S})^2Var\left(\hat{W}_g(\mathcal{S})\right)+2\sum_{\{S,T\}\subseteq \mathcal{P}}P(\mathcal{S})P(T) \text{Cov}\left(\hat{W}_g(\mathcal{S}),\hat{W}_g(T)\right)
\end{align*}
where,
\begin{align*}
Var\left(\hat{W}_1(\mathcal{S})\right)&=b^2_{s+1}\sigma^2_{1}(\s \cup \{k\})+b^2_{s}\sigma^2_{1}(\s)-2 b_{s+1}b_{s}\text{Cov}\Big(\hat{w}_1(\s \cup \{k\}),\hat{w}_1(\s) \Big) 
\end{align*}
Let us formally derive $\text{Cov}\big(\hat{w}_1(\s \cup \{k\}),\hat{w}_1(\s) \big) $, we have: 
\begin{align*}
\text{Cov}\bigl(\hat{w}_1(\s\cup\{k\}),\,\hat{w}_1(\s)\bigr)
&= b_2^2\,\text{Cov}\Bigl(\widehat{TPR}_{\mathcal{A}}(\s\cup\{k\}),\,\widehat{TPR}_{\mathcal{A}}(\s)\Bigr)\\[1mm]
&\quad + b_2b_1\,\Biggl\{
    \text{Cov}\Bigl(\widehat{TPR}_{\mathcal{A}}(\s\cup\{k\}),\,\widehat{TPR}_{1}(\s)-\widehat{TPR}_{2}(\s)\Bigr)\\[1mm]
&\quad\quad\quad + \text{Cov}\Bigl(\widehat{TPR}_{1}(\s\cup\{k\})-\widehat{TPR}_{2}(\s\cup\{k\}),\,\widehat{TPR}_{\mathcal{A}}(\s)\Bigr)
\Biggr\}\\[1mm]
&\quad + b_1^2\,\Big[
    \text{Cov}\Bigl(\widehat{TPR}_{1}(\s\cup\{k\}),\,\widehat{TPR}_{1}(\s)\Bigr)\\[1mm]
&\quad\quad + \text{Cov}\Bigl(\widehat{TPR}_{2}(\s\cup\{k\}),\,\widehat{TPR}_{2}(\s)\Bigr)
\Big]
\end{align*}
Similarly, we have
\begin{align*}
\text{Cov}\Bigl(\hat{w}_2(\s\cup\{k\}),\,\hat{w}_2(\s)\Bigr)
&= b_2^2\,\text{Cov}\Bigl(\widehat{TPR}_{\mathcal{A}}(\s\cup\{k\}),\,\widehat{TPR}_{\mathcal{A}}(\s)\Bigr)\\[1mm]
&\quad - b_2b_1\,\Biggl\{\text{Cov}\Bigl(\widehat{TPR}_{\mathcal{A}}(\s\cup\{k\}),\,\widehat{TPR}_{1}(\s)-\widehat{TPR}_{2}(\s)\Bigr)\\[1mm]
&\quad\quad\quad + \text{Cov}\Bigl(\widehat{TPR}_{1}(\s\cup\{k\})-\widehat{TPR}_{2}(\s\cup\{k\}),\,\widehat{TPR}_{\mathcal{A}}(\s)\Bigr)
\Biggr\}\\[1mm]
&\quad + b_1^2\,\Big[
    \text{Cov}\Bigl(\widehat{TPR}_{1}(\s\cup\{k\}),\,\widehat{TPR}_{1}(\s)\Bigr)\\[1mm]
&\quad\quad + \text{Cov}\Bigl(\widehat{TPR}_{2}(\s\cup\{k\}),\,\widehat{TPR}_{2}(\s)\Bigr)
\Big]
\end{align*}
Note that since $\widehat{TPR_g}=\frac{TP_g}{n_g^+}:=\frac{\sum_{i=1}^{n_g}I_{i,\s}}{\sum_{i=1}^{n_g}J_{i}}$, then
$$\text{Cov}\Bigl(\widehat{TPR}_{g}(\s\cup\{k\}),\,\widehat{TPR}_{g}(\s)\Bigr)=\frac{1}{({n_g^{+}})^2} \text{Cov}\Bigl(TP_{g}(\s\cup\{k\}),\,TP_{g}(\s)\Bigr)
$$
where $\text{Cov}\bigl(TP_{g}(\s\cup\{k\}),\,TP_{g}(\s)\bigr)$ is given by the following:
\begin{align*}
\text{Cov}\Bigl(TP_{g}(\s\cup\{k\}),\,TP_{g}(\s)\Bigl)
&=\mathbb{E}(TP_{g}(\s\cup\{k\})\cdot TP_{g}(\s))-\mathbb{E}(TP_{g}(\s\cup\{k\}))\mathbb{E}(TP_{g}(\s))\\
&=\mathbb{E}\Big(\sum_{i=1}^{n_g^+}I_{i,g,\s\cup\{k\}}\sum_{i=1}^{n_g^{+}}I_{i,g,\s}\Big)-\mathbb{E}\Big(\sum_{i=1}^{n_g^{+}}I_{i,g,\s\cup\{k\}}\Big)\mathbb{E}\Big(\sum_{i=1}^{n_g^{+}}I_{i,g,\s}\Big)\\
&=\sum_{i=1}^{n_g^{+}}\mathbb{E}\Big(I_{i,g,\s\cup\{k\}}I_{i,g,\s}\Big)+\sum_{\substack{i,j = 1 \\ i \neq j}}^{n_g^{+}}\mathbb{E}\Big(I_{i,g,\s\cup\{k\}}I_{j,g,\s}\Big)\\ &
-\mathbb{E}\Big(\sum_{i=1}^{n_g^{+}}I_{i,g\s\cup\{k\}}\Big)\mathbb{E}\Big(\sum_{i=1}^{n_g^{+}}I_{i,g,\s}\Big)\\
&=\sum_{i=1}^{n_g^{+}}\p(I_{i,g,\s\cup\{k\}}=1,I_{i,g,\s}=1)+\sum_{\substack{i,j = 1 \\ i \neq j}}^{n_g^{+}}\mathbb{E}\Big(I_{i,g,\s\cup\{k\}}\Big)\mathbb{E}\Big(I_{j,g,\s}\Big)\\
&-\sum_{i=1}^{n_g^{+}}\mathbb{E}\Big(I_{i,g,\s\cup\{k\}}\Big)\sum_{j=1}^{n_g^{+}}\mathbb{E}(I_{j,g,\s})\\
&=n_g^{+}\p(I_{g,\s\cup\{k\}}=1,I_{g,\s}=1)-n_g^{+}\Big(\p(I_{g,\s\cup\{k\}}=1)\p(I_{g,\s}=1)\Big)\\
&=n_g^{+}\p(I_{g,\s\cup\{k\}}=1,I_{g,\s}=1)-n_g^{+}\Big(p_g(\s\cup\{k\})p_g(\s))\Big)
\end{align*}

Let us set $TP_{\mathcal{A}} := TP_{1} + TP_{2}$, where $TP_{g}$ represents the true positives associated with $g$. It follows that
\begin{align*}
\text{Cov}\Big(
    \frac{TP_{\mathcal{A}}(\mathcal{S})}{n_{\mathcal{A}}^+},
    \,
    \frac{TP_{1}(\mathcal{S})}{n_1^+}
    - \frac{TP_{2}(\mathcal{S})}{n_2^+}
\Big)
&= \frac{1}{n_{\mathcal{A}}^+ n_1^+}\,
    \text{Cov}\big(
        TP_{\mathcal{A}}(\mathcal{S}),
        TP_{1}(\mathcal{S})
    \big) \\
&\quad - \frac{1}{n_{\mathcal{A}}^+ n_2^+}\,
    \text{Cov}\big(
        TP_{\mathcal{A}}(\mathcal{S}),
        TP_{2}(\mathcal{S})
    \big).
\end{align*}

Moreover, we have:
\begin{align*}
\text{Cov}\Big(
    TP_{\mathcal{A}}(\mathcal{S}\cup\{k\}),
    \, TP_{1}(\mathcal{S})
\Big)
&= n_1^{+}\,
    \mathbb{P}\big(
        I_{1,\mathcal{S}\cup\{k\}} = 1,\,
        I_{1,\mathcal{S}} = 1
    \big) \\
&\quad - n_1^{+}\,
    \mathbb{P}\big(
        I_{1,\mathcal{S}\cup\{k\}} = 1
    \big)\,
    \mathbb{P}\big(
        I_{1,\mathcal{S}} = 1
    \big) \\
\\
\text{Cov}\Big(
    TP_{\mathcal{A}}(\mathcal{S}\cup\{k\}),
    \, TP_{2}(\mathcal{S})
\Big)
&= n_2^{+}\,
    \mathbb{P}\big(
        I_{2,\mathcal{S}\cup\{k\}} = 1,\,
        I_{2,\mathcal{S}} = 1
    \big) \\
&\quad - n_2^{+}\,
    \mathbb{P}\big(
        I_{2,\mathcal{S}\cup\{k\}} = 1
    \big)\,
    \mathbb{P}\big(
        I_{2,\mathcal{S}} = 1
    \big).
\end{align*}

Similarly,
\begin{align*}
\text{Cov}\big(\hat{W}_g(\mathcal{S}), \hat{W}_g(T)\big)
&= \text{Cov}\Big(
    b_{s+1}\hat{w}_g(\mathcal{S}\cup\{k\}) - b_s \hat{w}_g(\mathcal{S}),
    \;
    b_{t+1}\hat{w}_g(T\cup\{k\}) - b_t \hat{w}_g(T)
\Big) \\
&= b_{s+1}b_{t+1}
    \text{Cov}\Big(
        \hat{w}_g(\mathcal{S}\cup\{k\}),
        \hat{w}_g(T\cup\{k\})
    \Big) \\
&\quad - b_{s+1}b_t
    \text{Cov}\Big(
        \hat{w}_g(\mathcal{S}\cup\{k\}),
        \hat{w}_g(T)
    \Big) \\
&\quad - b_{t+1}b_s
    \text{Cov}\Big(
        \hat{w}_g(T\cup\{k\}),
        \hat{w}_g(\mathcal{S})
    \Big) \\
&\quad + b_s b_t
    \text{Cov}\Big(
        \hat{w}_g(\mathcal{S}),
        \hat{w}_g(T)
    \Big).
\end{align*}

Under the null hypothesis that $\mathcal{C}^k_1(v^{(TPR)})=\mathcal{C}^k_2(v^{(TPR)})$, the asymptotic distribution of the standardized
test statistic is:
\begin{align*}
Z_k=\frac{\hat{\mathcal{C}}^k_1-\hat{\mathcal{C}}^k_2}{\sqrt{\widehat{Var}(\hat{\mathcal{C}}^k_1)+\widehat{Var}(\hat{\mathcal{C}}^k_2)-2\widehat{\text{Cov}}(\hat{\mathcal{C}}^k_1(v^{(TPR)}),\hat{\mathcal{C}}^k_2(v^{(TPR)}))}}\overset{d}{\longrightarrow} \mathcal{N}(0,1) 
\end{align*}

Now, it remains to compute $\text{Cov}(\hat{\mathcal{C}}^k_1(v^{(TPR)}),\hat{\mathcal{C}}^k_2(v^{(TPR)}))$. We have:
\begin{align*}
&\text{Cov}(\hat{\mathcal{C}}^k_1(v^{(TPR)}),\hat{\mathcal{C}}^k_2(v^{(TPR)}))
\notag \\ &=\sum_{\s}P(\mathcal{S})^2\text{Cov}\left(\hat{W}_1(\mathcal{S}),\hat{W}_2(\s)\right)+\sum_{\s }\sum_{T\neq \mathcal{S}}P(\mathcal{S})P(T)\text{Cov}\left(\hat{W}_1(\mathcal{S}),\hat{W}_2(T)\right)\   
\end{align*}
where
\begin{align*}
&\sum_{\s}P(\mathcal{S})^2\text{Cov}\left(\hat{W}_1(\mathcal{S}),\hat{W}_2(\s)\right)\\ 
&=\sum_{\s}P(\mathcal{S})^2\Big[\text{Cov}\big(b_{s+1}\hat{w}_1(\s\cup\{k\}) - b_s \hat{w}_1(\s) ,b_{s+1}\hat{w}_2(\s\cup\{k\}) - b_s \hat{w}_2(\s) \big)\Big]\\
&=\sum_{\s}P(\mathcal{S})^2\Big[b_{s+1}b_{s+1}\text{Cov}\big( \hat{w}_1(\s\cup\{k\}),\hat{w}_2(\s\cup\{k\}) \big)-b_{s+1}b_s\text{Cov}\big( \hat{w}_1(\s\cup\{k\}),\hat{w}_2(\s) \big)\\
&-b_{s+1}b_s\text{Cov}\big( \hat{w}_2(\s\cup\{k\}),\hat{w}_1(\s)\big)+b_sb_s\text{Cov}\big( \hat{w}_2(S),\hat{w}_1(\s)\big) \Big]\\
&=\sum_{\s}P(\mathcal{S})^2\ b_{s+1}^2\Big[b_2^2 Var\big(\widehat{TPR}_{\mathcal{A}}(\s\cup\{k\})\big)-b_1^2(Var(\widehat{TPR}_{1}(\s\cup\{k\}))\\
&-b_1^2 Var(\widehat{TPR}_{2}(\s\cup\{k\}))\Big]-b_{s+1}b_s\Big[b_2^2 \text{Cov}\big(\widehat{TPR}_{\mathcal{A}}(\s\cup\{k\}),\,\widehat{TPR}_{\mathcal{A}}(\s)\big)\\
&-b_1^2 \text{Cov}\big(\widehat{TPR}_{1}(\s\cup\{k\}),\,\widehat{TPR}_{1}(\s)\big)-b_1^2\text{Cov}\big(\widehat{TPR}_{2}(\s\cup\{k\}),\,\widehat{TPR}_{2}(\s)\big)\Big]\\
&-b_1b_2\text{Cov}\big(\widehat{TPR}_{\mathcal{A}}(\s\cup\{k\}),\,\widehat{TPR}_{1}(\s)\big)+b_1b_2\text{Cov}\big(\widehat{TPR}_{\mathcal{A}}(\s\cup\{k\}),\,\widehat{TPR}_{2}(\s)\big)\\
&+b_1b_2\text{Cov}\big(\widehat{TPR}_{1}(\s\cup\{k\}),\,\widehat{TPR}_{\mathcal{A}}(\s)\big)-b_1b_2\text{Cov}\big(\widehat{TPR}_{2}(\s\cup\{k\}),\,\widehat{TPR}_{\mathcal{A}}(\s)\big)\Big]\\
&-b_{s+1}b_s \Big[b_2^2 \text{Cov}\big(\widehat{TPR}_{\mathcal{A}}(\s\cup\{k\}),\,\widehat{TPR}_{\mathcal{A}}(\s)\big)\\
&-b_1^2 \text{Cov}\big(\widehat{TPR}_{1}(\s\cup\{k\}),\,\widehat{TPR}_{1}(\s)\big)-b_1^2\text{Cov}\big(\widehat{TPR}_{2}(\s\cup\{k\}),\,\widehat{TPR}_{2}(\s)\big)\Big]\\
&-b_1b_2\text{Cov}\big(\widehat{TPR}_{\mathcal{A}}(\s),\,\widehat{TPR}_{1}(\s\cup\{k\})\big)+b_1b_2\text{Cov}\big(\widehat{TPR}_{\mathcal{A}}(\s),\,\widehat{TPR}_{2}(\s\cup\{k\})\big)\\
&+b_1b_2\text{Cov}\big(\widehat{TPR}_{1}(\s),\,\widehat{TPR}_{\mathcal{A}}(\s\cup\{k\})\big)-b_1b_2\text{Cov}\big(\widehat{TPR}_{2}(\s),\,\widehat{TPR}_{\mathcal{A}}(\s\cup\{k\})\big)\Big]\\ 
&+b_{s}^2\Big[b_2^2 Var\big(\widehat{TPR}_{\mathcal{A}}(\s)\big)-b_1^2 Var(\widehat{TPR}_{1}(\s))-b_1^2 Var(\widehat{TPR}_{2}(\s))\Big]\\
&=\sum_{\s}P(\mathcal{S})^2 \ b_{s+1}^2\Big[b_2^2 Var\big(\widehat{TPR}_{\mathcal{A}}(\s\cup\{k\})\big)-b_1^2 Var(\widehat{TPR}_{1}(\s\cup\{k\}))\\
&-b_1^2 Var(\widehat{TPR}_{2}(\s\cup\{k\})\Big]-b_{s+1}b_s\Big[2b_2^2 \text{Cov}\big(\widehat{TPR}_{\mathcal{A}}(\s\cup\{k\}),\,\widehat{TPR}_{\mathcal{A}}(\s)\big)\\
&-2b_1^2 \text{Cov}\big(\widehat{TPR}_{1}(\s\cup\{k\}),\,\widehat{TPR}_{1}(\s)\big)-2b_1^2\text{Cov}\big(\widehat{TPR}_{2}(\s\cup\{k\}),\,\widehat{TPR}_{2}(\s)\big)\Big]\\
&+b_{s}^2\Big[b_2^2 Var(\widehat{TPR}_{\mathcal{A}}(\s))-b_1^2 Var(\widehat{TPR}_{1}(\s))-b_1^2 Var(\widehat{TPR}_{2}(\s))\Big]
\end{align*}
The same kind of computation is used for $\sum_{\s }\sum_{T\neq \mathcal{S}}P(\mathcal{S})P(T)\text{Cov}\big(\hat{W}_1(\mathcal{S}),\hat{W}_2(T)\big)$. Consequently, the null hypothesis is rejected at level $\alpha$ when $|Z_k|>z_{\alpha/2}$.
\end{proof}

\end{document}